\documentclass{article}

\usepackage[preprint,nonatbib]{neurips_2023}

\usepackage{amsmath}
\usepackage{algorithm}
\usepackage{algorithmic}

\usepackage[english]{babel}
\usepackage{amsthm}
\newtheorem{theorem}{Theorem}[section]

\newtheorem{definition}{Definition}[section]

\usepackage{multirow, makecell}
\usepackage{subfig}
\usepackage{enumitem}
\usepackage{graphicx}

\newcommand{\header}[1]{\vspace{1mm}\noindent\textbf{#1}.}
\newcommand{\headerl}[1]{\vspace{1mm}\noindent\textit{#1}.}

\usepackage[utf8]{inputenc} 
\usepackage[T1]{fontenc}    
\usepackage{hyperref}       
\usepackage{cleveref}
\usepackage{url}            
\usepackage{booktabs}       
\usepackage{amsfonts}       
\usepackage{nicefrac}       
\usepackage{microtype}      
\usepackage{xcolor}         

\title{Improving Retrieval-Augmented~Large Language Models via Data Importance Learning}

\author{
Xiaozhong Lyu$^{1}$ \quad Stefan Grafberger$^{2}$ \quad Samantha Biegel$^{1}$ \quad Shaopeng Wei$^{1}$ \\ \textbf{ Meng Cao$^{3}$ \quad Sebastian Schelter$^{2}$ \quad Ce Zhang$^{1}$} \\ $^{1}$ETH Zürich $^{2}$ University of Amsterdam  $^{3}$Apple
}

\begin{document}

\maketitle

\newcommand\datapoint{data point}
\newcommand\datapoints{data points}
\newcommand\gptthree{GPT-3.5}

\begin{abstract}
 Retrieval augmentation enables large language models to take advantage of external knowledge, for example on tasks like question answering and data imputation. However, the performance of such retrieval-augmented models is limited by the data quality of their underlying retrieval corpus. In this paper, we propose an algorithm based on \textit{multilinear extension} for evaluating the data importance of retrieved \datapoint{}s. 
 There are \textit{exponentially} many terms
 in the multilinear extension, and one 
 key contribution of this paper is a 
 \textit{polynomial time} algorithm that 
 computes \textit{exactly}, given a retrieval-augmented model with an additive utility function and a validation set, 
 the data importance of \datapoint{}s in the retrieval corpus using the multilinear extension of the model's utility function. 
 We further proposed an even more efficient $(\epsilon, \delta)$-approximation algorithm.
 Our experimental results illustrate that we can enhance the performance of large language models by only pruning or reweighting the retrieval corpus, without requiring further training. For some tasks, this even allows a small model (e.g., GPT-JT), augmented with a search engine API, to outperform GPT-3.5 (without retrieval augmentation). 
 Moreover, we show that weights based on multilinear extension can be computed efficiently in practice (e.g., in less than ten minutes for a corpus with 100 million elements).

\end{abstract}


\section{Introduction}

Large language models (LLMs) consisting of neural networks with billions of parameters and trained on vast quantities of unlabelled text are the basis of unprecented progress in natural language processing tasks~\cite{devlin2018bert,radford2018improving,2020t5,lewis2019bart}. With zero-shot or few-shot prompting, LLMs can be adopted for a wide range of diverse tasks, such as question answering~\cite{liang2022holistic} summarization~\cite{liang2022holistic, bhaskar2022zero} and data imputation~\cite{narayancan2022}.

\header{Drawbacks of large language models} LLMs, however, have two widely acknowledged disadvantages~\cite{alt2019fine, cost2020}. Firstly, despite their impressive capabilities, LLMs actually perform badly on tail entities~\cite{alt2019fine}, which they have not seen at training time or cannot remember due to limitations of the network capacity. The second drawback is that with the ever-growing number of model parameters, training, and fine-tuning costs are exploding as well. As a rough estimate, it costs \$80k - \$1.6m to train a 1.5 billion parameter language model~\cite{strubell2019energy,cost2020,yuan2022decentralized}. This makes it difficult to leverage LLMs for tasks that require regularly updated data or that regularly need to remove privacy-sensitive or copyright-protected data~\cite{bommasani2021opportunities}.

\header{Retrieval-augmented models} To address such problems, retrieval-augmented (RAG) models have recently been proposed~\cite{karpukhin2020dense, lewis2020retrieval, guu2020retrieval}. A typical retrieval-augmented model consists of two parts, a retriever $f_{ret}$ and a generator $f_{gen}$. Given a retrieval corpus $\mathcal{D}_{ret} = \{d_1, \cdots, d_M\}$, the retriever $f_{ret}$ retrieves $K$ \datapoints{} for an input $x_i$ as $f_{ret}(x_i, \mathcal{D}_{ret}) = \{d_{\alpha_{1}}, d_{\alpha_{2}}, ..., d_{\alpha_{K}}\}$. Here, $\alpha_k$ denotes the rank of each \datapoint{} in the retrieval corpus assigned by the retriever. The generator $f_{gen}$ then generates its prediction based on the input and the retrieved \datapoints{} as evidence $f_{gen}(x_i, f_{ret}(x_i, \mathcal{D}_{ret}))$. Recent research indicates that incorporating external knowledge into LLMs improves their performance for various tasks and allows them to easily adapt to new knowledge~\cite{siriwardhana2022improving,zamani2022retrieval}. 

\begin{figure}
  \centering
  \includegraphics[width=\textwidth]{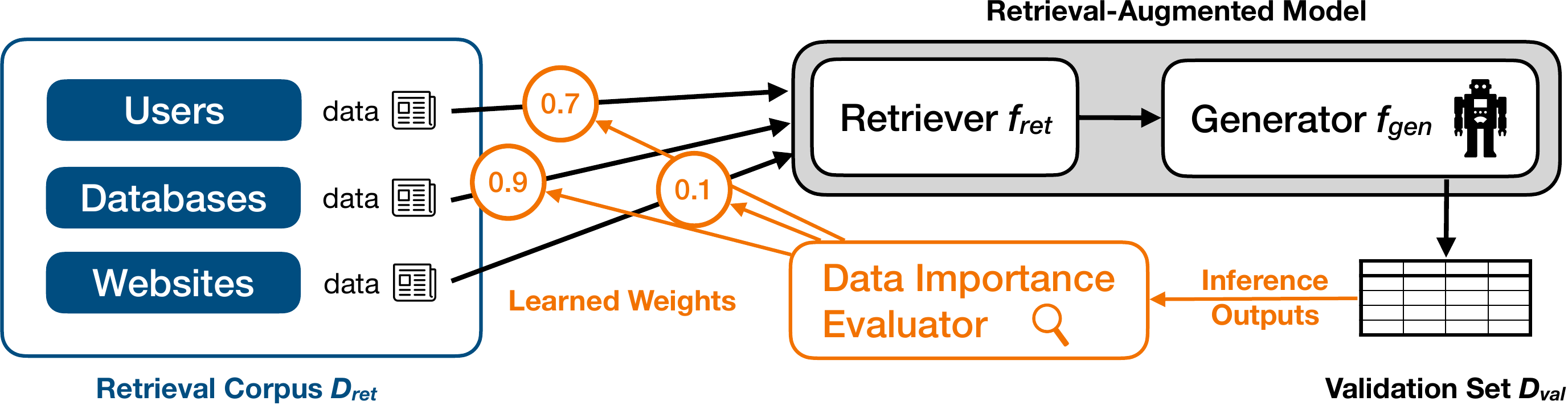}
  \caption{Data importance evaluation for retrieval-augmented models: The retriever $f_{ret}$ retrieves $K$ data points from the retrieval corpus $\mathcal{D}_{ret}$ and provides them to the answer generator $f_{gen}$. Our data importance evaluator learns weights for the data sources in the retrieval corpus based on the performance on a validation set $\mathcal{D}_{val}$. These weights are subsequently used to reweight or prune the data sources, and improve the model's performance without further training. 
  }
  \label{fig:workflow}
\end{figure}

\header{Impact of data quality on retrieval-augmented LLMs} The performance of retrieval-augmented models is highly limited by the quality of the retrieved \datapoints{}. For example, GPT-3 is able to give the correct answer ``\textbf{\textit{Frank Herbert}}'' to the question \textit{``Who is the author of Old Rambling House?''} with the help of a retrieved Wikipedia page ~\cite{enwiki:1140483446}, which contains the sentence \textit{"Old Rambling House is a short story by American science fiction author \textbf{Frank Herbert}."} However, it would with a high probability give the wrong answer if the retrieved page contained incorrect text such as \textit{``Old Rambling House is a short story by American science fiction author \textbf{J. R. R. Tolkien}.''}
Retrieval corpora are rarely clean in reality (especially if the underlying data comes from the web), and the origin of noise and errors in the data is difficult to track down~\cite{frenay2013classification, cothey2004web}. For example, according to recent estimates, $8.0\%$ to $38.5\%$ of labels in real-world datasets are corrupted~\cite{song2022learning}. In the domain of natural language processing, which relies on raw text, the rapidly growing number of use cases and an increasing amount of text have especially exacerbated data quality issues~\cite{cothey2004web}.

\header{Learning the data importance of retrieval sources} Given this data quality problem, we propose to improve retrieval-augmented models by learning the data importance of retrieval sources. Let $U\left( \cdot \right)$ be the utility function of a retrieval-augmented model with a validation set $\mathcal{D}_{val} = \{x_1, x_2, ..., x_N\}$, and let $\mathcal{D}_{ret} = \{d_1, \cdots, d_M\}$ be the underlying retrieval corpus of $M$ \datapoints{}. The performance of the model can be written as:
\begin{equation}
U(f_{gen}, f_{ret}, \mathcal{D}_{val}, \mathcal{D}_{ret}) := \sum_{x_i \subseteq \mathcal{D}_{val}}
     U\left( f_{gen}(x_i, f_{ret}(x_i, \mathcal{D}_{ret})) \right)      
\end{equation}
Our goal to is find a subset $\mathcal{S}$ of the retrieval corpus $\mathcal{D}_{ret}$ that maximizes the utility function $U(f_{gen}, f_{ret}, \mathcal{D}_{val}, \mathcal{S})$. We leave out $f_{gen}$, $f_{ret}$, and $\mathcal{D}_{val}$ from the notation and use $U(\mathcal{S})$ for readability. It is hard to solve this combinatorial optimization problem since it requires enumerating exponentially many possible subsets $\mathcal{S}$. One natural way is to change this problem to an optimization problem on continuous functions. Therefore, we define the \textit{multilinear extension} of the utility function as:
\begin{equation}
\tilde{U}(w_1,\cdots,w_M) := \sum_{\mathcal{S} \subseteq \mathcal{D}_{ret}}
     U\left( \mathcal{S} \right)
       \underbrace{\prod_{d_i \in \mathcal{S}} w_i
       \prod_{d_i \not \in \mathcal{S}} (1 - w_i)}_{P[\mathcal{S}]}
\label{eq:main}
\end{equation}
Here,  $P[\mathcal{S}]$ denotes the probability of the sampled retrieval corpus $\mathcal{S}\subseteq \mathcal{D}_{ret}$ based on the weights $w_1,\cdots,w_M$. Our goal is to find the optimal weights $w_1,\cdots,w_M$ that maximize the multilinear extension of the utility function:
\begin{equation}
    \max_{w_1,\cdots,w_M \in [0, 1]} \tilde{U}(w_1,\cdots,w_M)
\end{equation}
The optimal weights can be found with textbook optimization methods like gradient descent. This, however, requires enumerating exponentially many sample sets, making the problem infeasible in practice. We tackle this challenge with the following main contributions of this paper: 

\begin{itemize}[leftmargin=*]

   \item We present an efficient algorithm to compute weights for a large family (but not all) of retrieval-augmented models with additive utility functions. Our algorithm has \textit{polynomial} time complexity and does not depend on the retrieval corpus size~(Sections~\ref{Exact_MLE}~\&~\ref{ep_MLE}), even
   given that there are \textit{exponential} many terms in Equation~\ref{eq:main}.

    \item We introduce an efficient estimation algorithm to compute the $(\epsilon,\delta)$-approximation of weights for a large family of retrieval-augmented models~(\Cref{ed_MLE}).

    \item We experimentally demonstrate that retrieval augmentation and data evaluation based on multilinear extension improve the performance of large language models in question answering and data imputation tasks. The experiments demonstrate that with external retrieval knowledge, small language
    models can yield comparable performance to large language models. Furthermore, our evaluation shows that weights based on multilinear extension can identify noisy data and help models adapt to new sources of knowledge~(\Cref{experiment3}).
    
    \item Our implementation of the algorithm illustrates that weights based on multilinear extension can be calculated very fast in practice, even for a large corpus with 100 million \datapoints{}~(\Cref{experiment4}). 

    \item We provide the source code of our implementation and experiments under \url{https://github.com/amsterdata/ragbooster}.
\end{itemize}    

\section{Algorithms for Deriving Gradients}

We can find the optimal weights for the multilinear extension of the utility function via computing the gradient of a particular weight $w_i$ based on a validation set $\mathcal{D}_{val}$:

\begin{equation}
\label{gradient}
\begin{aligned}
\frac{\partial \tilde{U}}{\partial w_i}
 & =
\sum_{\mathcal{S} \subseteq \mathcal{D}_{ret}\backslash d_i}{\left( 
      U(\mathcal{S} \cup \{d_i\})
     - U(\mathcal{S}) \right)
   \cdot P[\mathcal{S}]} \\
& =
\sum_{x_{val} \in \mathcal{D}_{val}} \frac{1}{|\mathcal{D}_{val}|} \cdot \sum_{\mathcal{S} \subseteq \mathcal{D}_{ret}\backslash d_i}
   \underbrace{\left( 
      U_{x_{val}}(\mathcal{S} \cup \{d_i\})
     - U_{x_{val}}(\mathcal{S}) \right)
    \cdot P[\mathcal{S}]}_{G(x_{val},\ w_i)}\\
& = \frac{1}{|\mathcal{D}_{val}|} \cdot \sum_{x_{val} \in \mathcal{D}_{val}}
G(x_{val},\ w_i)
\end{aligned}
\end{equation}

\header{Infeasability of a naive implementation} However, computing the gradients in~\Cref{gradient} is challenging. A naive implementation would have to {\em enumerate all possible subsets} $\mathcal{S}$ for each validation tuple $x_{val} \in \mathcal{D}_{val}$ to compute the contribution of this subset $\mathcal{S}$ to the gradient value $G(x_{val},\ w_i)$. Such a naive implementation is infeasible in practice due to its inherent exponential time complexity.

\header{Efficient weight computation for retrieval-augmented models} As discussed before, we focus on a specific family of machine learning models, called retrieval-augmented (RAG) models. Retrieval-augmented models benefit from locality: the predictions of retrieval-augmented models for an input sample are only determined by the Top-$K$ closest  \datapoint{}s in the retrieval corpus and the answer generator. Combined with additive utility functions (which are common for both classical KNN and state-of-the-art RAG models), this allows us to efficiently compute exact gradients within polynomial time complexity
(\Cref{Exact_MLE} and \Cref{sec:impl}). In \Cref{ep_MLE}, we show that we only have to consider a small subset of \datapoints{} for each validation tuple and that the time complexity only depends on $K$ instead of the retrieval corpus size $M$ if we apply an $\epsilon$-approximation. Finally, we propose an $(\epsilon,\delta)$-approximation algorithm in \Cref{ed_MLE} to calculate gradients for general utility functions.

\subsection{Exact Gradient Calculation for Models with an Additive Utility Function}
\label{Exact_MLE}

A textbook $K$-nearest neighbor classifier and many state-of-the-art retrieval-augmented models~\cite{lewis2020retrieval} can be viewed as models with additive utility functions. In this section, we present a polynomial time complexity algorithm to compute the exact gradient of the weights of the \textit{multilinear extension} of the utility function. We follow existing work~\cite{jia2019efficient} to define the additive utility function of a retrieval-augmented model as:
\begin{equation}
U_{x_{val}}(\mathcal{S}) = \frac{1}{K} \sum_{k=1}^{\min{(K,|\mathcal{S}|)}}{U_{x_{val}}(f_{gen}(d_{\alpha_k^{x_{val}}(\mathcal{S})}))}
\end{equation}
Here, $\alpha_k^{x_{val}}(\mathcal{S})$ represents the index of the \datapoint{}, which is the $k$th closest to $x_{val}$ among all the \datapoint{s} retrieved by $f_{ret}$ from $\mathcal{S}$. From now on, we abbreviate $\alpha_k^{x_{val}}(\mathcal{S})$ to $\alpha_k$. $U_{x_{val}}(f_{gen}(d_{\alpha_k^{x_{val}}(\mathcal{S})}))$ denotes the utility function for the output generated based on the validation tuple $x_{val}$ and the single \datapoint{} $d_{\alpha_k^{x_{val}}}$. We assume that the possible values of $U\left(\cdot\right)$ function are within a countable finite set $\mathcal{V}$, where $|\mathcal{V}| = V$, and leave out $f_{gen}$ from the notation for readability in the following. In this scenario, we can provide an algorithm with PTIME time complexity in~\Cref{ap_Exact_MLE}. The overall time complexity of the algorithm is $\mathcal{O}{\left(N\cdot(M \log{M} + M K^2 + M K V )\right)}$

\subsection{$\epsilon$-approximation Algorithm for Calculating Exact Gradient Values}
\label{ep_MLE}

The overall time complexity for computing gradients for models with an additive utility function is $\mathcal{O}{\left(N\cdot(M \log{M} + M K^2 + M K V )\right)}$. 
In this section, we show that if we are allowed to do $\epsilon$-approximations, we can significantly speed up the calculation of the gradients $\frac{\partial \tilde{U}}{\partial w_i}$. We will only introduce the main idea here, leaving the details in~\Cref{ap_ep_MLE}.

\begin{theorem}
\label{computeep}
If we calculate the $\epsilon$-approximation $\hat{G}(x_{val},\ w_i)$ for the each $G(x_{val},\ w_i)$, we can get the $\epsilon$-approximation for $\frac{\partial \tilde{U}}{\partial w_i}$ as the average of $\hat{G}(x_{val},\ w_i)$.
\end{theorem}
\begin{proof}See \Cref{pcomputeep}.\end{proof}

Our next step is to detail how to compute the $\epsilon$-approximation for $G(x_{val},\ w_i)$.
One observation is that the absolute value $G(x_{val},\ w_i)$ is bounded by the sum of the probabilities of the \datapoints{} $d_i$ in the $K$-nearest neighbor set of $x_{val}$. Notice that for a \datapoint{} with a lower rank, the probability of it being in the $K$-nearest neighbor set is smaller. Therefore we can define the boundary point $d_b$ of the retrieval corpus.  

\begin{definition}{(\textbf{Boundary Point})}
Given a validation tuple $x_{val}$ and the retrieval corpus $\mathcal{D}_{ret} = \{d_1, \cdots, d_M\}$ ranked with respect to $x_{val}$, the \textbf{boundary point $d_b$} is the \datapoint{} with the highest rank in the sorted corpus such that any \datapoint{} that has a lower rank than $d_b$ has a probability less than $\epsilon$ to be in the $K$-nearest neighbor set of $x_{val}$. 
\end{definition}
In practice, after we rank the corpus with respect to a validation tuple, we can use binary search to find the boundary point. After we find this boundary point $d_b$, we can use 0 as the $\epsilon$-approximation for the gradient for \datapoint{s} with a lower rank as $\hat{G}(x_{val},\ w_i) = 0$ for $i \in \{b, ..., M\}$. It is because the probability of those \datapoint{s} being in the $K$-nearest neighbor set is less than $\epsilon$. In the following, we will show the approximation for \datapoint{s} with a higher rank. 


\begin{theorem}
\label{discard}
Given the validation tuple $x_{val}$, the retrieval corpus $\mathcal{D}_{ret} = \{d_1, ..., d_M\}$, the boundary point $d_b$, and the weights $W = \{w_1, ..., w_M\}$, if we have an algorithm $\mathcal{A}$ to calculate the $G(x_{val}, w_i) = \mathcal{A}(x_{val},\mathcal{D}_{ret}, W )$, then $\hat{G}(x_{val},\ w_i) = \mathcal{A}(x_{val},\{d_1, ..., d_b\},\{w_1, ..., w_b\})$ is the $\epsilon$-approximation for $G(x_{val}, w_i)$.
\end{theorem}
\begin{proof}See~\Cref{pdiscard}\end{proof}

From~\Cref{discard}, we can compute the $\epsilon$-approximation for every \datapoint{}~by discarding the outlier points $\{d_b, d_{b+1}, ..., d_M\}$. This reduces the time complexity from $\mathcal{O}{\left(N\cdot(M \log{M} + M K^2 + M K V )\right)}$ to $\mathcal{O}{\left(N\cdot(B \log{B} + B K^2 + B K V )\right)}$ where $B$ is the rank of the boundary point. 

\begin{theorem}
\label{BoundD}
If the value of all $w_i$ is greater than a certain constant $\lambda$, then the index of the boundary point $B$ is $\mathcal{O}(K)$.
\end{theorem}
\begin{proof}See~\Cref{pBoundD}\end{proof}

The above theorem shows that if all weights $W$ are greater than a certain constant, the scale of $B$ is only related to $K$ instead of the size of the retrieval corpus $M$. It means that even though we may have millions of \datapoints{} in the retrieval corpus, we only have to consider $O(K)$ \datapoints{} with the highest rank for a validation tuple. The overall time complexity for computing the approximate gradients for models with additive utility functions is $\mathcal{O}{\left(N\cdot(K \log{K} + K K V )\right)} = \mathcal{O}{\left(N\cdot K^2 \cdot V\right)}$. This significantly speeds up their computation.

\subsection{$(\epsilon, \delta)$-approximation Algorithm for Models with General Utility Functions}
\label{ed_MLE}

Next, we provide a solution for efficiently approximating gradients for retrieval-augmented models with a general utility function. According to what we proposed in the previous section, for every validation tuple, we can find the boundary point of the retrieval corpus. When a point has a smaller rank score than the boundary point, the epsilon approximation is 0. Using the Markov chain Monte Carlo method, we can calculate an approximation of the gradients for a retrieval-augmented model with a general utility function. In light of the fact that 0 is the approximate value for most points, we only need to perform MCMC on a small number of data points. Detailed proofs and algorithms are provided in \cref{aped_MLE}.

\subsection{Projected Gradient Descent for Weights on a Data Source Level}
\label{sec:impl}

\header{Exact gradients for a grouped retrieval corpus} In the previous section, we introduced the algorithm for calculating gradients for weights for the multilinear extension of the utility function. We also proved that each validation tuple only contributes gradients to a small part of the retrieval corpus. A further problem is how to evaluate the quality of the \datapoints{} which are not retrieved for the validation tuples. In real-world ML applications, a retrieval corpus is commonly generated from various data sources. For example, \datapoints{} in the retrieval corpus may come from the same labeler, the same websites, or the same database. As a consequence, we can evaluate data quality at this level, which we call the source level. This has the additional advantage that we do not have to inspect every data point before identifying if the data is useful. We formulate the corresponding problem as follows. Given a series of data sources for the retrieval corpus $\mathcal{O}_{ret} = \{o_1, o_2, ..., o_M\}$, the generated retrieval corpus can be represented as a function of these sources ${D}_{ret} = \bigcup_{i=1}^{M}{f_{source}(o_i)}$. We detail how to compute the exact gradient of the weights for the K-Nearest Neighbor classifier and a grouped corpus in~\Cref{ap_group_MLE}. The time complexity of the algorithm is $\mathcal{O}\left(N \cdot T^2\cdot M^2\right)$, where $T$ is the size of the generated retrieval corpus. 

\header{Projected gradient descent for weights on a grouped corpus} In general, given the retrieval corpus and the validation set, we can use a textbook batch gradient descent algorithm to find the optimal weights for the \datapoint{}s in the retrieval corpus. From the previous paragraph, we can see that computing the exact gradient values for a grouped retrieval corpus with several data sources can be computationally expensive. Therefore, we propose a projected gradient descent algorithm to efficiently learn the optimal weights for retrieval corpus generated from data sources. Given the generated retrieval corpus represented as a function of the sources $\{o_1, o_2, ..., o_M\}$, ${D}_{ret} = \bigcup_{i=1}^{M}{f_{source}(o_i)}$, we assign a weight to each \datapoint{} in the generated retrieval corpus ${D}_{ret}$. Suppose there are $m_i$ \datapoint{}s in $f_{source}(o_i)$, we assign the weights $\{w_{i,1}.w_{i,2}, ..., w_{i, m_i}\}$ to each \datapoint{} in $f_{source}(o_i)$. The original optimization problem can be relaxed to a constrained optimization problem as detailed below:
\begin{equation}
\begin{aligned}
    &\max_{w_{1,1},\cdots,w_{M,m_M} \in [0, 1]} \tilde{U}(\ w_{1,1},..., w_{M,m_M}) \\
    &\begin{array}{r@{\quad}r@{}l@{\quad}l}
    s.t. &w_{1,1} = w_{1, 2} = \cdots &= w_{1, m_1} \\
         &w_{2,1} = w_{2, 2} = \cdots &= w_{2, m_2} \\
         & \cdots\\
         &w_{M,1} = w_{M, 2} = \cdots &= w_{M, m_M} \\
    \end{array}
\end{aligned}
\end{equation}
To find the optimum of this function, we use the existing algorithm for a non-grouped corpus to compute the gradient of the weight for each $w_{i,j}$. After we update the parameters using the gradients, we project the updated $w_{i,j}$ to satisfy the constraints by computing $\hat{w_i} = \frac{1}{\alpha} \sum{w_{i,j}}$ and set every $w_{i,j}$ to $\hat{w_i}$. Therefore, we can utilize the algorithm introduced in~\Cref{ep_MLE} to calculate the gradient and then compute the average. For retrieval-augmented models with additive utility functions, the time complexity becomes $\mathcal{O}{\left(N\cdot K^2 \cdot V + T\right)}$.


\section{Experimental Evaluation}
\label{experiment}

We conduct a series of experiments for question answering and data imputation tasks. We confirm in~\Cref{experiment1} that retrieval augmentation enhances the performance of large language models.~\Cref{experiment2} and~\Cref{experiment3} show that the multilinear extension weights help us identify noisy/incorrect data in the retrieval corpus, and that pruning or reweighting the retrieval corpus accordingly improves performance without the need to fine-tune the underlying model. The runtime of the algorithm is examined in~\Cref{experiment4}, where we showcase that the weights can be computed very fast in practice. We provide the source code of our implementation and experiments under \url{https://github.com/schelterlabs/retrieval_importance}.

\header{Datasets and tasks} For {\em question answering}, we leverage the \texttt{WikiFact}~\cite{liang2022holistic} dataset, in which questions are extracted from Wikipedia pages using relation pairs. The answer to each question in this dataset can be found on Wikipedia. For example, for the relation "author", a question is "The author of Nimmer on Copyright is ?". We filter out relations with less than 500 questions and use each of the remaining 70 relations as a separate downstream task. In {\em data imputation}, the task is to predict missing values of a relational table~\cite{narayancan2022}. We experiment with two common benchmark datasets for this task: \texttt{restaurant}, where the \texttt{city} column of a table about restaurants must be imputed, and \texttt{buy}, where we have to impute the \texttt{manufacturer} column in a table about electronics products. For each experimental run on a question answering or data imputation task, we randomly split the dataset into validation dataset and test dataset with an equal number of tuples. We repeat this for 64 different random seeds, and report the mean accuracy. For the zero-shot baselines in the imputation tasks, we use the prompts suggested in \cite{narayancan2022}.

\header{Language models} We leverage the language model GPT-JT~\cite{tay2022transcending, tay2022unifying} with 6 billion parameters, which we enhance with retrieval augmentation. As a reference, we compare this to the language model ``text-davinci-003'' (to which we refer to as \gptthree{}) from OpenAI's commercial GPT-3.5 family~\cite{openai}. For both language models, we generate predictions with zero-shot or few-shot prompting, without further fine-tuning.

\header{Retrieval augmentation} We leverage the Microsoft Bing search engine~\cite{Bing} to generate a retrieval corpus for each task. We create a query from each validation/test sample (e.g., the question to answer) and retrieve the first 50 websites together with their textual snippets provided by Bing as retrieved \datapoints{} for the sample. We sort these \datapoints{} according to the ranking score provided by Bing. We create a few-shot prompt from each retrieved \datapoint{}, and generate an answer for the corresponding validation sample via GPT-JT. We decide on the final prediction via a majority vote using the generated answers from the top-$K$ websites.

\header{Reweighting or pruning the retrieval corpus} In experiments which reweight or prune the retrieval corpus  based on multilinear extension weights, we proceed as follows. We choose K = 10 and set the initial weight to 0.5. We group the retrieved websites by their domain name, and run the projected gradient descent algorithm from \Cref{sec:impl} for 50 iterations with a learning rate of 500 on the validation dataset to compute the optimal weights. Next, for reweighting, we compute the expectation of the accuracy on the test set by randomly sampling the retrieved \datapoint{}s 32 times based on the learned weights to form the retrieval corpus. For pruning, we remove retrieved \datapoint{}s with a learned weight below a certain threshold (tuned on the validation set) before computing predictions on the test set via majority vote. We use the leave-one-out (LOO) error as a \textbf{baseline} to refine the retrieval corpus. We compute the change in accuracy for the removal of each individual data source and finally remove all data sources with a LOO error below a certain threshold (tuned on the validation set) before computing predictions on the test set. 

\subsection{Benefits of Retrieval Augmentation}
\label{experiment1}

\headerl{Experimental setup} The aim of this experiment is to confirm the well-known fact that retrieval augmentation alone already enhances the performance of language models. We compare the performance of \gptthree{} without retrieval augmentation to the performance of GPT-JT with retrieval augmentation on the question answering and data imputation tasks.

\headerl{Results and discussion} The results for question answering are shown in~\Cref{Srawgptretri}. The mean accuracy of \gptthree{} over all 70 relations is 0.33, which outperforms the mean accuracy of 0.21 achieved by GPT-JT standalone. However, retrieval augmentation raises the mean accuracy of GPT-JT to 0.33, making it competitive with the 30x larger \gptthree{}. The smaller model even outperforms the larger model in the majority of relations (39 out of 70, detailed results available in~\Cref{gpt_jt_w_retrieve}). We encounter the analogous behavior for data imputation in \Cref{tab:imputation}, where retrieval augmentation (\textsc{vanilla}) makes the small GPT-JT model competitive with the 30x larger \gptthree{} model, and even outperforms it on both datasets.

\begin{table*}[t]
\caption{Average accuracy for question answering on \texttt{Wikifact}. A small language model with retrieval augmentation and learned multilinear extension weights outperforms a large model with 30 times more parameters.}
\subfloat[Benefits of retrieval augmentation.]{
	\scriptsize
	\begin{sc}
	\begin{tabular}{cccc|c}
	\toprule
	\multirowcell{2}{GPT-JT\\(6B)} &
	\multicolumn{3}{c}{GPT-JT (6B) w/ Retrieval} & \multirowcell{2}{\gptthree{}\\(175B)} \\
	\cmidrule(r){2-4}
	   &  K = 1 & K = 10 & K = 50 &\\
	\midrule
	 0.214 & 0.332 & \underline{0.333} & 0.293&\textbf{0.339} \\
	\bottomrule
	\end{tabular}
	\end{sc}
	\label{Srawgptretri}
}
\quad
\subfloat[Benefits of weight-based reweighting and pruning.]{
		\scriptsize
		\begin{sc}
		\begin{tabular}{cccc|c}
		\toprule
		\multicolumn{4}{c}{GPT-JT (6B) w/ Retrieval} & \multirowcell{2}{\gptthree{}\\(175B)} \\
		\cmidrule(r){1-4}
		   vanilla & + loo & + reweight & + prune&\\
		\midrule
	 0.333 & 0.358 & \underline{0.380} & \textbf{0.392} & 0.339\\
		\bottomrule
		\end{tabular}
		\end{sc}
		\label{SURLreweighting}
}
\end{table*}


\begin{figure}
  \centering
  \includegraphics[width=\textwidth]{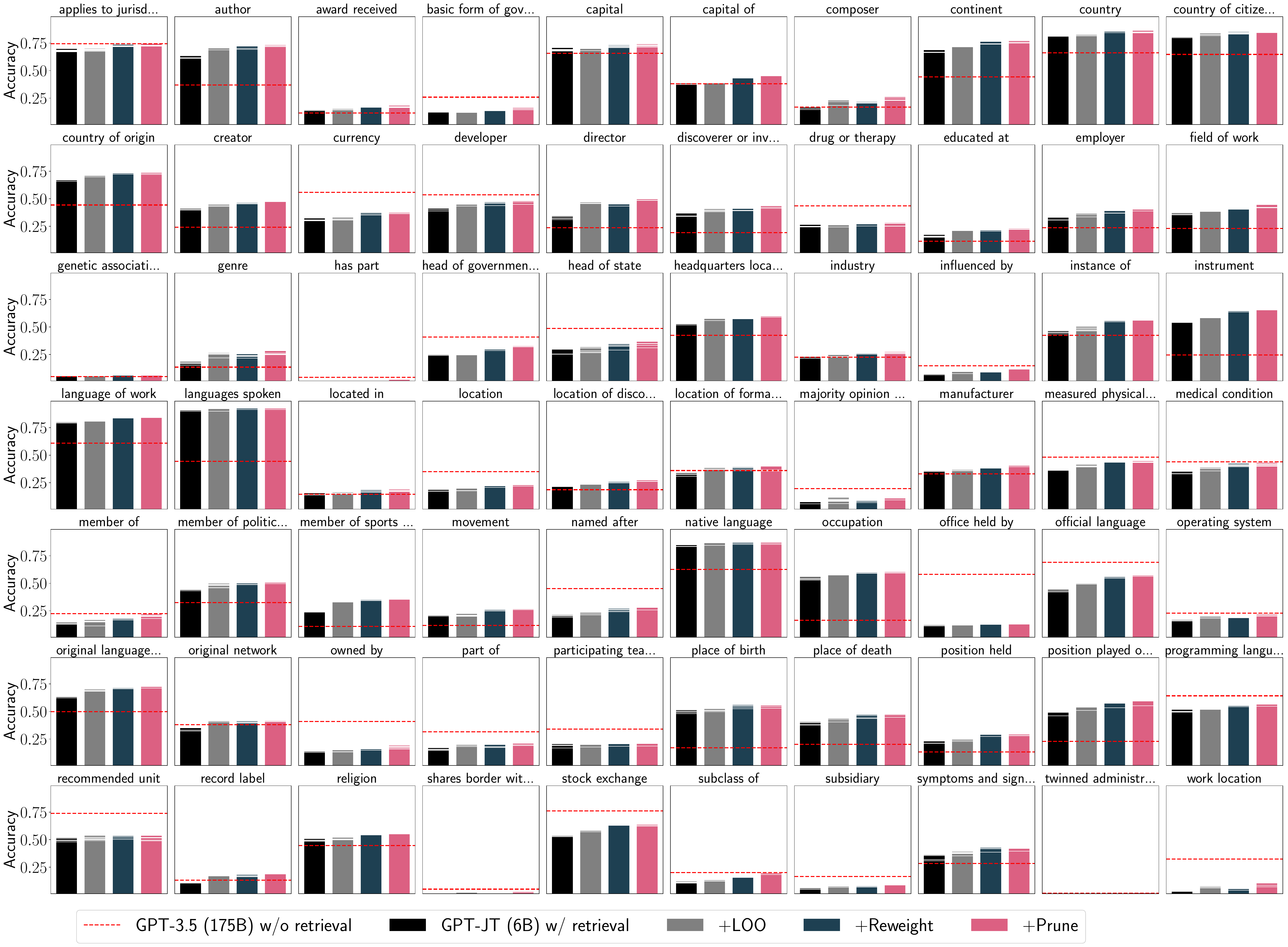}
  \caption{Accuracy for question answering on the 70 relations from \texttt{WikiFact}.}
  \label{fig:result}
\end{figure}

\subsection{Improving Performance with Multilinear Extension Weights}
\label{experiment2}

\headerl{Experimental setup} Next, we showcase that pruning or reweighting the retrieval corpus based on multilinear extension weights importance improves performance without having to fine-tune the underlying model. We group the retrieved websites by domain and refine the corpus as detailed earlier. 

\headerl{Results and discussion} The results for question answering are shown in Table~\ref{SURLreweighting} (detailed results in\Cref{fig:result} and \Cref{MLE_wikifact}), and confirm that reweighting and pruning using the learned weights increases test accuracy. The mean accuracy of the GPT-JT model retrieval augmentation increases from 33.3\% to 37.7\% after pruning (and to 36.9\% after reweighting) using the multilinear extension weights, and it clearly outperforms the state-of-the-art \gptthree{} model with 175 billion parameters. In all 70 relations, the performance improved using multilinear extension weights by removing 71.5\% of the retrieval corpus on average. Analogously, we find that the performance in the data imputation tasks is improved by pruning based on the learned weights importance as well (\Cref{tab:imputation}). For both datasets, the smaller model outperforms \gptthree{} by more than 5\% in test accuracy. These results confirm that the performance of retrieval-augmented models can be further optimized by evaluating the quality and reliability of real-world data sources in their underlying corpus.


\begin{table*}[t]
	\caption{Average accuracy for data imputation on \texttt{buy} and \texttt{restaurant}.}
	\label{tab:imputation}
	\vskip 0.15in
	\begin{center}
		\scriptsize
		\begin{sc}
			\begin{tabular}{lccccc|c}
				\toprule
				\multirow{2}*{Dataset}  &
				\multirowcell{2}{GPT-JT \\(6B)} &	
				\multicolumn{4}{c}{GPT-JT (6B) w/ retrival}&
				\multirowcell{2}{\gptthree{}\\(175B)} \\
				\cmidrule(r){3-6}
				 & & vanilla& +loo & +reweight & +prune &\\
				\hline
						Buy  & 0.102 & 0.789 &  0.808 & \textbf{0.815}& \underline{0.813}& 0.764\\
						Restaurant  & 0.030 &  0.746 &0.756&\underline{0.760}& \textbf{0.761}& 0.463\\ 
				\bottomrule
			\end{tabular}
		\end{sc}
	\end{center}
	\vskip -0.1in
\end{table*}

\subsection{Mitigating the Impact of Noise in the Retrieval Corpus}
\label{experiment3}

\begin{table*}[t]
\caption{Accuracy improvements for GPT-JT (6B) with retrieval augmentation on a noisy corpus.}
\label{Sdirtycorpus}
\vskip 0.15in
\begin{center}
\scriptsize
\begin{sc}
\begin{tabular}{ccccr}
\toprule
\multirow{2}*{CLEAN CORPUS} &
\multicolumn{4}{c}{DIRTY CORPUS} \\ 
\cmidrule{2-5}
  & vanilla & + loo & + reweight & + prune\\
\midrule
\underline{0.333} & 0.270 & 0.311 & 0.330 & \textbf{0.335} \\
\bottomrule
\end{tabular}
\end{sc}
\end{center}
\vskip -0.1in
\end{table*}

\headerl{Experimental setup} The aim of the following experiment is to demonstrate how the learned weights assist us with mitigating the impact of noise in the retrieval corpus. To achieve this, we manually inject noise into the retrieval corpus of the question-answering task as follows. We create five copies of the retrieval corpus for each question with noise levels ranging from $0\%$ to $80\%$ (resulting in around 250 retrieved websites per question, of which $40\%$ are corrupted). To inject noise, we randomly replace the correct answer in the retrieved websites with an incorrect one according to the noise level. Then, for each copy, we randomly split the corpus into ten sources according to rank. Now we have $5 \cdot 10$ different sources in total with different noise levels. We expect a performance drop when using the dirty corpus and aim to demonstrate how data evaluation can help us restore performance. 

\headerl{Results and discussion} As shown in \Cref{Sdirtycorpus}, the performance drops from 33.3\% on the clean corpus to 27.0\% on the dirty corpus with injected noise. Using the leave-one-out error to remove noise sources improves performance to 31.1\%. Both reweighting and pruning using learned weights drastically improve the performance on the dirty corpus and both enhance the performance by over 33.0\% on the dirty corpus. Pruning even results in a better performance of 33.5\% compared to the clean corpus without pruning. The results show that even if we are faced with a situation where nearly half of the retrieval corpus is noisy, multilinear extension weights can help the model reach performance comparable to the clean corpus.

%

  \begin{minipage}{\textwidth}
  \begin{minipage}[b]{0.61\textwidth}
    \centering
    \scriptsize
    \begin{sc}
    \begin{tabular}{ccccc|c}
    \toprule
    
    \multirowcell{2}{GPT-JT (6B)\\w/~ Retrieval} &
    \multicolumn{4}{c}{GPT-JT (6B) w/~ Retrieval + Fabricated Data} & \multirowcell{2}{\gptthree{} \\(175B)} \\ 
    \cmidrule{2-5}
     &  vanilla & +loo & +reweight & +prune&\\
    \midrule
     0.333 & 0.382 & 0.399 & \underline{0.410} & \textbf{0.418}  & 0.339\\
    \bottomrule
    \end{tabular}
    \end{sc}
    \vspace{1cm}
    \captionof{table}{Accuracy impact of additional fabricated data sources for question answering on \texttt{Wikifact}.}
    \label{Sfakewikireweighting}
    \end{minipage}
    \hfill
    \begin{minipage}[b]{0.35\textwidth}
      \centering
      \includegraphics[width=\textwidth]{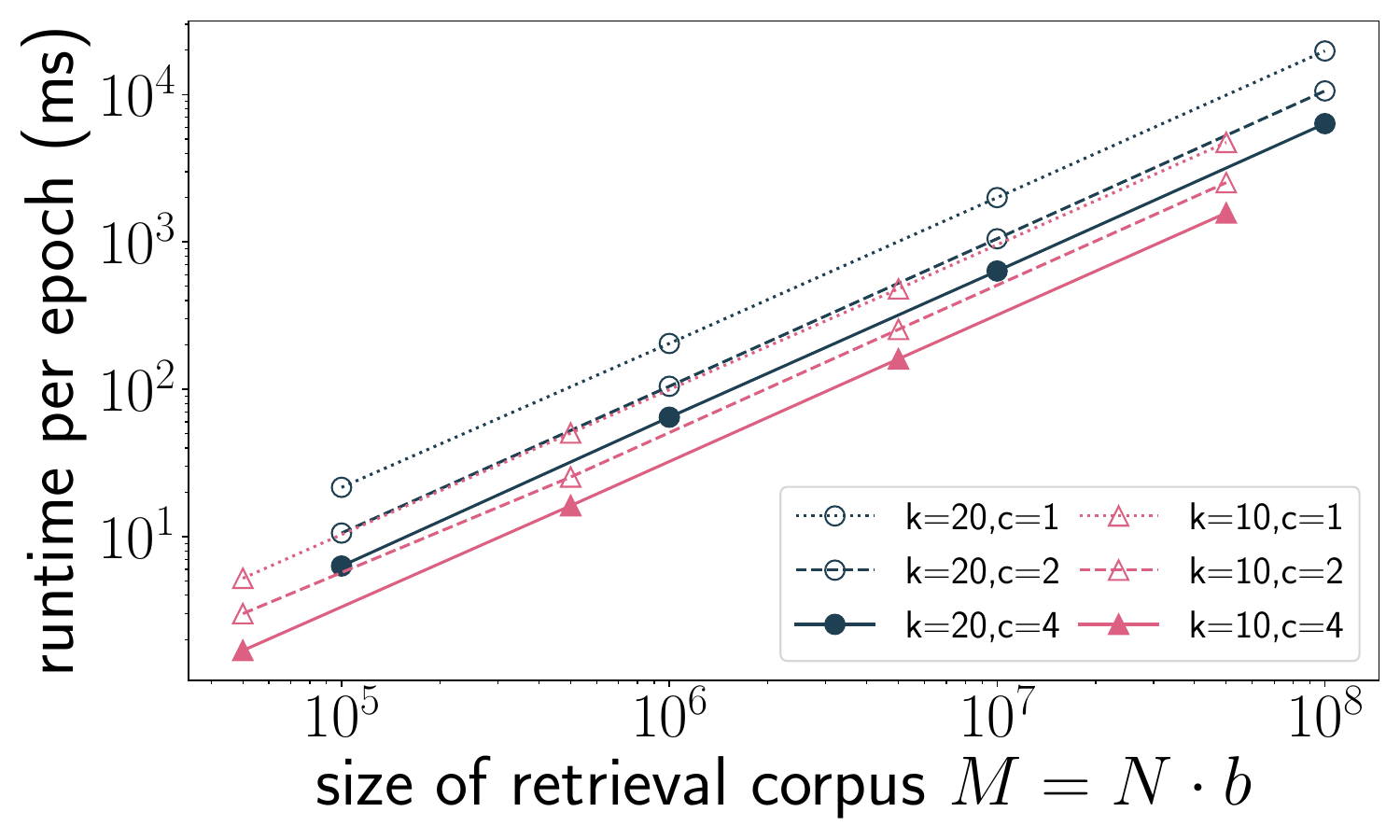}
      \captionof{figure}{Runtime per epoch on corpora with up to 100M elements.}
      \label{fig:runtimes}	
    \end{minipage}    
  \end{minipage}

\subsection{Handling Auto-Generated Data Sources in the Retrieval Corpus}
\label{exp:adversarial}

\headerl{Experimental setup} Next, we illustrate how learned weights allows us to handle new sources in the retrieval corpus for question answering. We manually generate five synthetic Wikipedia pages for each question using the OpenAI ``text-davinci'' generator. We adopt the real Wikipedia pages as a few-shot example, add the fabricated sources to the retrieval corpus and give them the highest rank among the websites. We aim to show that when new knowledge is added to the corpus, the learned weights help us to utilize the sources based on their quality. 

\headerl{Results and discussion} \Cref{Sfakewikireweighting} shows the results of this experiment. We find that adding fabricated Wikipedia pages to the corpus increases the accuracy from 33.3\% to 38.2\%. This is due to the fact that the OpenAI model itself can reach 33.9\% and most Wikipedia pages contain the correct information if the model memorizes the answer. We see, however (e.g., for the relation "place of death"), that adding generated Wikipedia pages will decrease the performance from 38.3\% to 33.8\%.  Using LOO to prune the retrieval corpus improves performance by 39.9\% on average. Reweighting or pruning using the learned multilinear extension weights achieves the highest accuracy of 41.0\% and 41.8\%, improving the performance on the corpus without fabricated Wikipedia sources. The results show that the learned weights help the model to easily adapt to new knowledge sources without further training.

\subsection{Computational Performance}
\label{experiment4}

\headerl{Experimental setup} Finally, we illustrate that the weights can be computed very fast in practice. For that, we implement our approach in Rust (with a Python frontend), and apply several performance optimizations to the code such as parallelization, memory pre-allocation and re-use, operator fusion, and predication~\cite{chentvm2018,neumann2011efficiently}. We run the experiments on consumer hardware (a machine with a four-core Intel i7-8569U CPU @2.80GHz, 16GB of RAM, and MacOS 12.6). We measure the runtime of our implementation on three relations from the \texttt{Wikifact} dataset (``author'', ``place-of-birth'', ``currency''), which contain 1,700-2,700 questions each, with 50 corresponding retrieved answers per question. We additionally run experiments on a synthetic retrieval corpus whose size $M = N \cdot b$ we scale up from 50,000 to 100,000,000 (with a validation set size $N$ from 1,000 to 1000,000 times $b=[50, 100]$ retrieved \datapoints{} per sample). We run each configuration with one, two, and four threads, repeat each run seven times, and measure the mean execution time per epoch. 

\headerl{Results and discussion} For the relations from \texttt{WikiFact}, a gradient update only takes between two and four milliseconds. We plot the results for the synthetic corpus in \Cref{fig:runtimes}. The x-axis is the size of the retrieval corpus $M = N \cdot b$ (size $N$ of the validation set times the number of retrieved \datapoints{} per sample $b$) and the y-axis denotes the mean runtime in milliseconds with a logarithmic scale. We see that with all four cores, we can finish an epoch for corpora with up to 10 million elements with a sub-second runtime. Even for the largest corpus with 100 million elements, an epoch can be conducted in 6.3 seconds on consumer hardware. Furthermore, we find that the runtime grows linearly with the size of the retrieval corpus and that our implementation easily benefits from parallelism when multiple cores are utilized. This showcases that data refinement using multilinear extension weights is computationally cheaper than model fine-tuning, which (in many cases) has to conduct an expensive backpropagation of errors through the underlying model.



\section{Conclusion}

We presented efficient algorithms to compute the optimal weights that maximize the multilinear extension of the utility function and use them to refine the retrieval corpus for retrieval-augmented large language models. Overall, our results illustrate that the learned weights are a powerful metric for evaluating the quality of the retrieval corpus and that retrieval-augmented models can be enhanced by only pruning the retrieval corpus without further training the underlying model. Furthermore, the weights can be computed efficiently even for a large retrieval corpus, and allow us to easily adapt predictions in cases where new sources are added to the retrieval corpus. 


\newpage
\appendix



\section{Exact Gradient Calculation for Models with an Additive Utility Function}

\label{ap_Exact_MLE}

We will first introduce two building blocks to help calculate the gradient:
\begin{definition}{(\textbf{Subset Probability})}
Given the retrieval corpus $\mathcal{D}_{ret} = \{d_1, ..., d_M\}$ and the weights $W = \{w_1, ..., w_M\}$, \textbf{subset probability} $P_k(a, b)$ is the sum of the probability of subsets with a size of $k$ from $\mathcal{D}^{\prime} = \{d_a, d_{a+1},...,d_b\}$.
\begin{equation}
P_k(a, b) = \sum_{|\mathcal{S}| = k, \mathcal{S}\subseteq\mathcal{D}^{\prime}} \prod_{d_i \in \mathcal{S}} w_i
       \prod_{d_i \not \in \mathcal{S}} (1 - w_i)
\end{equation}
\end{definition}
We only use the subset probability values $P_{\cdot}{(1, \cdot)}$ and $P_{\cdot}{(\cdot, M)}$. We compute these subset probability values within $\mathcal{O}(MK)$ time complexity, leveraging previous work on efficiently computing Poisson-binomial distribution values~\cite{hong2013computing}.

\begin{definition}{(\textbf{Boundary Value Probability})}
Given the validation tuple $x_{val}$, the retrieval corpus $\mathcal{D}_{ret} = \{d_1, ..., d_M\}$, the weights $W = \{w_1, ..., w_M\}$ and the possible value set $\mathcal{V}$ of the utility function, the \textbf{boundary value probability} $B_k(i, e)$ is the sum of the probability of all subsets $\mathcal{S}$ sampled from $\mathcal{D}^{\prime} = \{d_i, d_{i+1},...,d_M\}$ whose $k$-th element $d_{\alpha_k{(\mathcal{S})}}$ is evaluated as $e$.

\begin{equation}
B_k(i, e) = \sum_{U_{x_{val}}(d_{\alpha_k{(\mathcal{S})}}) = e, \mathcal{S}\subseteq\mathcal{D}^{\prime}} \prod_{d_i \in \mathcal{S}} w_i
       \prod_{d_i \not \in \mathcal{S}} (1 - w_i)
\end{equation}
\end{definition}

This term can be calculated via dynamic programming:

\begin{equation}
\label{DPBVP}
B_k(i, e)=\left\{
\begin{array}{lr}
B_k(i+1, e) * (1 - w_i) +  
\mathbb{I}\{ U_{x_{val}}(d_{i})=e\} 
* w_i & k=1\\
B_k(i+1, e) * (1 - w_i) +  B_{k-1}(i+1, e) * w_i & k>1 
\end{array}
    \right.
\end{equation}

To make Equation~(\ref{DPBVP}) correct for every $k \in [1, K]$, $i \in [1, M]$ and $e \in \mathcal{V}$, we initialize the boundary value to $B_k(M+1, e)=0$ for $k \in[1, K], e \in \mathcal{V}$. The time complexity of computing the boundary value probability using the above equation is $\mathcal{O}{(M K V)}$.

With these two building blocks, we are able to calculate the exact value of $G(x_{val},\ w_i)$. We examine two situations.

(1) $|\mathcal{S}|<K$

In this case, the size of sampled retrieval corpus $\mathcal{S}$ is smaller than $K$. Therefore, including the \datapoint{} $d_i$ in $\mathcal{S}$ does not expel any \datapoint{} from the $K$-nearest neighbor set. Thus, 
\begin{equation} 
      U_{x_{val}}(\mathcal{S} \cup \{d_i\})
     - U_{x_{val}}(\mathcal{S}) = 
     \frac{U_{x_{val}}(d_{i})}{K}
\end{equation}
The sum of the probability of the respective subsets $\mathcal{S}$ equals the probability of selecting subsets with sizes less than $K$ from $\{ d_1,...,d_{i-1},d_{i+1},...,d_M \}$. The gradient for these sampled subsets $\mathcal{S}$ can be written as:
\begin{equation}
\begin{aligned}
 G_1(x_{val},\ w_i) 
& = \sum_{|S| <K,\mathcal{S} \subseteq \mathcal{D}_{ret}\backslash d_i}
    \frac{U_{x_{val}}(d_{i})}{K}
    \cdot P[\mathcal{S}] \\
& = \frac{U_{x_{val}}(d_{i})}{K} 
    \cdot \sum_{k^{\prime} = 0}^{K-1} \sum_{j = 0}^{k^{\prime}} P_{j}{(1, i-1)}
    \cdot P_{k^{\prime}-j}{(i+1, M)}
\end{aligned}
\end{equation}
The time complexity of computing all $G_1(x_{val},\ w_i)$ using the above equation is $\mathcal{O}{(M K^2)}$.

(2) $|\mathcal{S}| \geq K$

In this scenario, adding \datapoint{} $d_i$ to the sampled corpus $\mathcal{S}$ expels \datapoint{} $d_{\alpha_K(S)}$ from the $K$-nearest neighbor set. The corresponding difference in the utility function is:
\begin{equation}
U_{x_{val}}(\mathcal{S} \cup \{d_i\})
     - U_{x_{val}}(\mathcal{S}) = 
     \frac{U_{x_{val}}(d_{i}) - U_{x_{val}}(d_{\alpha_K(S)})}{K}    
\end{equation}
The gradient contributed by the corresponding sampled subsets can be calculated by enumerating the \datapoints{} that would be expelled from the $K$-nearest neighbor set. Suppose $d_{k^{\prime}}$ is the one to be expelled, then the sum of the probability of the corresponding subsets $\mathcal{S}$ equals selecting $K$ data points from $\{ d_1,...,d_{i-1},d_{i+1},...,d_{k^{\prime}} \}$. Therefore, the sum of the gradient can be written as:
\begin{equation}
\begin{aligned}
 G_2(x_{val},\ w_i) 
= &\sum_{|S| \geq K,\mathcal{S} \subseteq \mathcal{D}_{ret}\backslash d_i}
   \frac{U_{x_{val}}(d_{i}) - U_{x_{val}}(d_{\alpha_K(S)})}{K}
    \cdot P[\mathcal{S}] \\
= &\sum_{e\in\mathcal{V}} 
\frac{U_{x_{val}}(d_{i}) - e}{K} \cdot  \sum_{j = 0}^{K-1} P_{j}{(1, i-1)}
    \cdot B_{K-j}{(i+1, e)}
\end{aligned}
\end{equation}

The time complexity of computing all $G_2(x_{val},\ w_i)$ using the above equation is $\mathcal{O}{(M K V)}$.
The exact gradient values $G(x_{val},\ w_i)$ can be computed by the sum of $G_1(x_{val},\ w_i)$ and $ G_2(x_{val},\ w_i)$. The detailed algorithm is shown in Algorithm~\ref{alg:MCforExMLEBVP}. The overall time complexity of the algorithm is $\mathcal{O}{\left(N\cdot(M \log{M} + M K^2 + M K V )\right)}$

\begin{algorithm}[tb]
   \caption{Exact Gradient Calculation for Models with an Additive Utility Function}
   \label{alg:MCforExMLEBVP}
\begin{algorithmic}
   \STATE {\bfseries Input:} \\
   $\mathcal{D}_{ret} = \{d_1, ..., d_M\}$, retrieval corpus; \\
   $\mathcal{D}_{val} = \{x_1,\cdots ,x_N\}$, validation set; \\
   $W = \{w_1, ..., w_M\}$, weights of \datapoint{s};\\
   \STATE {\bfseries Output:} \\
   $\{g_1,\cdots,g_M\}$, gradients of weights;\\
   \STATE {$\{g_1, ..., g_M\} \leftarrow 0$}
   \FOR{ $x_{val} \in \mathcal{D}_{val}$}
   \STATE $\{d_{\pi_1},\cdots,d_{\pi_M}\} \leftarrow$ SortByRankingScore($\mathcal{D}_{ret},x_{val}$)
   \STATE $P \leftarrow$ ComputeSubsetProb$(W, \pi)$
   \STATE $B,\mathcal{V} \leftarrow$ ComputeBVProb$(W, \pi, \mathcal{D}_{ret}, x_{val})$
   \FOR{$i \leftarrow 1 \ to \ M$}
   \STATE $\mu_1 \leftarrow
   \frac{U_{x_{val}}(d_{\pi_i})}{K} \cdot \frac{1}{N}$
   \FOR{$k^{\prime} \leftarrow 0 \ to \ K-1$}
   \FOR{$j \leftarrow 0 \ to \ k^{\prime}$}
   \STATE $g_{\pi_i}\leftarrow g_{\pi_i} + \mu_1 \cdot P_{j}{(1, i-1)} \cdot P_{k^{\prime}-j}{(i+1, M)}$ 
   \ENDFOR
   \ENDFOR
   \FOR{$e \in \mathcal{V}$}
   \STATE $\mu_2 \leftarrow \frac{U_{x_{val}}(d_{\pi_i}) - e}{K} \cdot \frac{1}{N}$
   \FOR{$j \leftarrow 0 \ to \ K-1$}
   \STATE $g_{\pi_i}\leftarrow g_{\pi_i} + \mu_2 \cdot P_{j}{(1, i-1)}
    \cdot B_{K-j}{(i+1, e)}$ 
   \ENDFOR
   \ENDFOR
   \ENDFOR
   \ENDFOR
\end{algorithmic}
\end{algorithm}

\section{$\epsilon$-approximation Algorithm for Calculating Exact Gradient Value}
\label{ap_ep_MLE}

The overall time complexity for computing gradients for models with an additive utility function is $\mathcal{O}{\left(N\cdot(M \log{M} + M K^2 + M K V )\right)}$. 
In this section, we show that if we are allowed to do $\epsilon$-approximations, we can significantly speed up the calculation of the gradients.

\begin{theorem}
\label{ap_computeep}
If we calculate the $\epsilon$-approximation $\hat{G}(x_{val},\ w_i)$ for the each $G(x_{val},\ w_i)$, we can get the $\epsilon$-approximation for $\frac{\partial \tilde{U}}{\partial w_i}$ as the average of $\hat{G}(x_{val},\ w_i)$.
\end{theorem}
\begin{proof}See \Cref{pcomputeep}.\end{proof}

Our next step is to detail how to compute the $\epsilon$-approximation for $G(x_{val},\ w_i)$. The variable $\phi_{\mathcal{S}, x_{val}}\left(\ d_i  \right) = (U_{x_{val}}(\mathcal{S} \cup \{d_i\}) - U_{x_{val}}(\mathcal{S}))$ equals zero if $d_i$ is not in the $K$-nearest neighbor set of $\mathcal{S} \cup \{d_i\}$. This is due to the fact that adding the \datapoint{} $d_i$ to the corpus will not change the \datapoints{} retrieved by the model. Assuming that the utility function value is within the range of $[0, 1]$, \Cref{gradient} can be written as:
\begin{equation}
\label{bound_gradient}
\begin{aligned}
 \left| G(x_{val},\ w_i) \right| 
& = \left| \sum_{\mathcal{S} \subseteq \mathcal{D}_{ret}\backslash d_i}
\mathbb{I}
   \left\{ d_i\in \operatorname{top_K}(\mathcal{S} \cup \{d_i\})\right\} \cdot
   \underbrace{\phi_{\mathcal{S}, x_{val}}\left(\ d_i  \right)}_{\in [-1, 1]}
    \cdot P(\mathcal{S}) \right| \\
& \leq \sum_{\mathcal{S} \subseteq \mathcal{D}_{ret}\backslash d_i}
\mathbb{I}
   \left\{ d_i \in \operatorname{top_K}(\mathcal{S} \cup \{d_i\})\right\} \cdot
    P[\mathcal{S}]
\end{aligned}
\end{equation}
From Equation~\ref{bound_gradient} we can see that the absolute value $G(x_{val},\ w_i)$ is bounded by the sum of the probabilities of the \datapoints{} $d_i$ in the $K$-nearest neighbor set of $x_{val}$. The probability of \datapoint{} $d_i$ to be in the $K$-nearest neighbor set equals the probability of less than $(K-1)$ points with higher ranks appearing in $\mathcal{S}$. The latter can be modeled by a Poisson-binomial distribution. Suppose that the retrieval corpus $\{d_1, d_2, \cdots, d_M\}$ is ranked with respect to $x_{val}$, then the gradient can be bounded via the Chernoff bound for $\mu_i > K-1$, where $\mu_{i}=\sum_{k=0}^{i-1}w_k$:
\begin{equation}
\begin{aligned}
\left| G(x_{val},\ w_i) \right| & \leq P[{\bf PB}(w_1, w_2, \cdots, w_{i-1}) \leq K-1]
&  \leq 
\exp(-\frac{(\mu_{i}-K+1)^2}{2\mu_{i}})
\end{aligned}
\end{equation}
Notice that for a \datapoint{} with a lower rank, the probability of it being in the $K$-nearest neighbor set is smaller. Therefore we can define the boundary point $d_b$ of the retrieval corpus. 

\begin{definition}{(\textbf{Boundary Point})}
Given a validation tuple $x_{val}$ and the retrieval corpus $\mathcal{D}_{ret} = \{d_1, \cdots, d_M\}$ ranked with respect to $x_{val}$, the \textbf{boundary point $d_b$} is the \datapoint{} with the highest rank in the sorted corpus such that any \datapoint{} that has a lower rank than $d_b$ has a probability less than $\epsilon$ to be in the $K$-nearest neighbor set of $x_{val}$. 
\end{definition}
In practice, after we rank the corpus with respect to a validation tuple, we can use binary search to find the boundary point.
\begin{equation}
\min_{b}{\left[\left(\exp(-\frac{(\mu_{b}-K+1)^2}{2\mu_{b}}) < \epsilon \right) \land \left( \mu_{b} > K-1 \right) \right]} 
\end{equation}
After we find this boundary point $d_b$, we can use 0 as the $\epsilon$-approximation for the gradient for \datapoint{s} with a lower rank. $\hat{G}(x_{val},\ w_i) = 0$ for $i \in \{b, ..., M\}$, because the probability of those \datapoint{s} being the $K$-nearest neighbor is less than $\epsilon$. In the following, we detail the approximation for \datapoint{s} with a higher rank. 


\begin{theorem}
\label{ap_discard}
Given the validation tuple $x_{val}$, the retrieval corpus $\mathcal{D}_{ret} = \{d_1, ..., d_M\}$, the boundary point $d_b$, and the weights $W = \{w_1, ..., w_M\}$, if we have an algorithm $\mathcal{A}$ to calculate the $G(x_{val}, w_i) = \mathcal{A}(x_{val},\mathcal{D}_{ret}, W )$, then $\hat{G}(x_{val},\ w_i) = \mathcal{A}(x_{val},\{d_1, ..., d_b\},\{w_1, ..., w_b\})$ is the $\epsilon$-approximation for $G(x_{val}, w_i)$.
\end{theorem}
\begin{proof}See~\Cref{pdiscard}\end{proof}

From~\Cref{ap_discard}, we can compute the $\epsilon$-approximation for every \datapoint{}~by discarding the outlier points $\{d_b, d_{b+1}, ..., d_M\}$. This reduces the time complexity from $\mathcal{O}{\left(N\cdot(M \log{M} + M K^2 + M K V )\right)}$ to $\mathcal{O}{\left(N\cdot(B \log{B} + B K^2 + B K V )\right)}$ where $B$ is the rank of the boundary point. 

\begin{theorem}
\label{ap_BoundD}
If the value of all $w_i$ is greater than a certain constant $\lambda$, then the index of the boundary point $B$ is $\mathcal{O}(K)$.
\end{theorem}
\begin{proof}See~\Cref{pBoundD}\end{proof}

The above theorem shows that if all weights $W$ are greater than a certain constant, the scale of $B$ is only related to $K$ instead of the size of the retrieval corpus $M$. It means even though we may have millions of \datapoints{} in the retrieval corpus, we only have to consider $O(K)$ passages with the highest rank to the validation tuple. The overall time complexity for computing the approximate gradients for weights of models with additive utility functions is $\mathcal{O}{\left(N\cdot(K \log{K} + K K V )\right)} = \mathcal{O}{\left(N\cdot K^2 \cdot V\right)}$. This significantly speeds up their computation.

\section{Exact Gradients for a Grouped Retrieval Corpus}
\label{ap_group_MLE}
In this section, we will compute the exact gradient of the weights for the K-Nearest Neighbor classifier assuming that the retrieval corpus is generated from multiple data sources. We can see from~\Cref{gradient} that the key component of computing the accurate gradient value is computing $G(x_{val},\ w_i)$. Inspired from previous work~\cite{karlavs2022data}. We can simplify the equation as follow:

\begin{equation}
\label{23}
\begin{aligned}
G(x_{val},\ w_i)
& = \sum_{t, t^{\prime} \in \mathcal{D}_{ret}} \cdot \sum_{\gamma, \gamma^{\prime} \in \Gamma} \cdot u_{\Delta}\left(\gamma, \gamma^{\prime}\right) \cdot \omega_{t, t^{\prime}}\left( \gamma, \gamma^{\prime}, i, x_{val}\right)
\end{aligned}
\end{equation}

The idea of Equation~(\ref{23}) is to enumerate which \datapoint{} $t$ in the generated retrieval corpus is the $k$th nearest neighbor $\alpha_K{(f_{source}(\mathcal{S}))}$ of a sampled subset $f_{source}(\mathcal{S})$. Added the \datapoint{s}~from the source $f_{source}(o_i)$ to the retrieval corpus may expel more than one \datapoint{} from the original $K$-nearest neighbor set. Therefore, we also enumerate which new \datapoint{} $t^{\prime}$ is the $\alpha_K{(f_{source}(\mathcal{S}\cup\{o_i\}))}$. 

The $\operatorname{tally}_{v,t}$ operator returns the number of data points with a similarity score greater than $t$ with the utility function value $v$. The $\operatorname{tally}_{t}{\mathcal{S}} = \left( \operatorname{tally}_{v_1,t}{\mathcal{S}}, \cdots, \operatorname{tally}_{v_{|\mathcal{V}|},t}{\mathcal{S}} \right)$ returns a tally vector $\gamma \in \Gamma \subset \mathbb{N}^{|\mathcal{V}|}$ consisting of tailed occurrences of each possible utility function value $v \in \mathcal{V}$ of $\alpha_K{(f_{source}(\mathcal{S}))}$. Let $\Gamma$ be the set of all possible tally vectors. Enumerating the label tally vectors allows us to easily calculate the difference in utility function value after adding the data source $o_i$ to the retrieval corpus by \[u_{\Delta}\left(\gamma, \gamma^{\prime}\right) = \sum_{v\in \mathcal{V}}{\frac{\gamma_v \cdot v}{K}} - \sum_{v\in \mathcal{V}}{\frac{\gamma^{\prime}_v \cdot v}{K}}\] 

Inspired by~\cite{karlavs2022data}, we associate a binary variable $a_i \in \mathcal{A}$ to every data source $o_i$ to represent the sampled dataset. We can define value assignments $z: \mathcal{A} \rightarrow \mathbb{B}$ to determine whether a data source is in the sampled dataset. By setting $z(a_i) = 0$, we expel $o_i$ from the sampled data sources for the retrieval corpus. By setting $z(a_i) = 1$, we include $o_i$ in the sampled data sources for the retrieval corpus. Let $\mathcal{Z}_{\mathcal{A}}$ be all possible value assignments. We can change counting the possibility of sampled datasets to counting the possibility of value assignments. The $\omega_{t, t^{\prime}}\left( \gamma, \gamma^{\prime}, i, x_{val}\right)$ is defined below:
\begin{equation}
\begin{aligned}
\omega_{t, t^{\prime}}\left( \gamma, \gamma^{\prime}, i, x_{val}\right)&= 
\sum_{z \in \mathcal{z}_{A \backslash\left\{a_i\right\}}} 
\underbrace{\prod_{z(a_j) \neq 0 } w_j
       \prod_{z(a_j) = 0} (1 - w_j)}_{P(v)} \\
& \cdot \mathbb{I}\left\{t=\alpha_K \left(\mathcal{D}_{ret}\left[z\left[a_i \leftarrow 0\right]\right]\right)\right\} \cdot \mathbb{I}\left\{t^{\prime}=\alpha_K \left(\mathcal{D}_{ret}\left[z\left[a_i \leftarrow 1\right]\right]\right)\right\} \\
& \cdot \mathbb{I}\left\{\gamma=\operatorname{tally}_t \left( \mathcal{D}_{ret}\left[z\left[a_i \leftarrow 0\right]\right]\right)\right\} \cdot \mathbb{I}\left\{\gamma^{\prime}=\text{tally}_{t^{\prime}} \left( \mathcal{D}_{ret}\left[z\left[a_i \leftarrow 1\right]\right]\right)\right\}
\end{aligned}
\end{equation}



We define the evaluation of data sources similarly to~\cite{karlavs2022data}:

\begin{equation}
\operatorname{eval}_z(j):= \begin{cases}(\mathbf{0}, \mathbf{0}), & \text { if } j=M+1, \\ (\mathbf{0}, \mathbf{0})+\operatorname{eval}_z\left(j+1\right) & \text { if } z(a_j)=0, \\ \left(\operatorname{tally}_{t}{(o_j)}, \operatorname{tally}_{t^{\prime}}{(o_j)}\right)+\operatorname{eval}_z\left(j+1\right) & \text { if } z(a_j)=1 .\end{cases}
\end{equation}

To count the sum of the probability of valid value assignments, we define the count function as:
\begin{equation}
\operatorname{count}_e(j):=\sum_{z \in \left\{z \in \mathcal{z}_{A} \mid \operatorname{eval}_z(j)=e\right\}}{\prod_{z(a_i) = 1, i\geq j} w_i
       \prod_{z(a_i) = 0, i\geq j} (1 - w_i)}
\end{equation}

$\operatorname{count}_e(j)$ can be computed by a dynamic programming algorithm as:

\begin{equation}
\begin{aligned}
\operatorname{count}_e(j) = \operatorname{count}_{e}\left(j+1\right) * (1 - w_j)+ \operatorname{ count }_{e- \left(\operatorname{tally}_{t}{(o_j)}, \operatorname{tally}_{t^{\prime}}{(o_j)}\right)}\left(j+1\right) * w_j    
\end{aligned}
\end{equation}


We initialize the value with $\operatorname{count}_0(M+1) = 1$ and $\operatorname{count}_{e}{(M+1)} = 0$, where $e$ in $\Gamma \times \Gamma$. Then we can compute the value as follows:

\begin{equation}
\omega_{t, t^{\prime}}\left(\gamma, \gamma^{\prime},i, x_{val}\right)=\operatorname{count}_{\left(\gamma, \gamma^{\prime}\right)}\left(1\right)
\end{equation}


If we assume the parameter $K$ and the possible values of the label tally vector are constants, the time complexity of the algorithm is $\mathcal{O}\left(N \cdot T^2\cdot M^2\right)$, where $T$ is the size of the generated retrieval corpus. 

\section{$(\epsilon, \delta)$-approximation Algorithm for General Utility Function}
\label{aped_MLE}
In this section, we provide a general solution for efficiently approximating gradients with a general utility function. In practice, we can use the Monte Carlo Method to approximate the gradient. Based on Equation~\ref{gradient}, we can adapt the Monte Carlo method to get an $(\epsilon, \delta)$-approximation for each $G(x_{val},\ w_i)$. For each validation tuple $x_{val}$, we randomly sample a retrieval corpus $\mathcal{S}$ from the $\mathcal{D}_{ret}\backslash d_i$ to compute the estimation for $G(x_{val},\ w_i)$.

\begin{theorem}
If we can calculate the $(\epsilon, \delta)$-approximation for the each $G(x_{val},\ w_i)$, we can get the $(\epsilon, \delta)$-approximation for $\frac{\partial \tilde{U}}{\partial w_i}$.    
\end{theorem}
\begin{proof}
To get the $(\epsilon, \delta)$-approximation for $\frac{\partial \tilde{U}}{\partial w_j}$, we set the $\epsilon^{\prime} = \epsilon, \delta^{\prime} = \frac{\delta}{|\mathcal{D}_{ret}|}$. Suppose we can get the $(\epsilon^{\prime}, \delta^{\prime})$-approximation $\hat{G}(x_{val},\ w_i)$ for the each $G(x_{val},\ w_i)$, then we calculate $\hat{g}_i$ as 
\begin{equation}
\hat{g}_i = \frac{1}{|\mathcal{D}_{val}|} \cdot \sum_{x_{val} \in \mathcal{D}_{val}}
\hat{G}(x_{val},\ w_i)    
\end{equation}
and $\hat{g}_i$ is the $(\epsilon, \delta)$-approximation for $\frac{\partial \tilde{U}}{\partial w_i}$. Each of the $|\mathcal{D}_{ret}|$ steps of the algorithm has at most a $\delta^{\prime}$ chance of failure. The union bound then bounds the total chance of failure by $\delta^{\prime} \cdot |\mathcal{D}_{ret}| = \delta$. Analogous to~\Cref{pcomputeep}, the difference between $\hat{g}_i$ and $\frac{\partial \tilde{U}}{\partial w_i}$ can be bound by $\epsilon$. Therefore, we have obtained the $(\epsilon, \delta)$-approximation for $G(x_{val},\ w_i)$.

\end{proof}

However, the naive implementation of the approximation algorithm is time-consuming. We want to do fewer estimation steps without sacrificing accuracy. We can describe the improved algorithm of computing  $G(x_{val},\ w_i)$ as follows:
\begin{enumerate}
    \item \textbf{Initialization}
    Given a validation tuple $x_{val}$ and the retrieval corpus $\mathcal{D}_{ret}$, we first rank the data points with respect to the validation tuple.
    \item \textbf{Filtering the outlier points} 
    We use a binary search to find the boundary point. Then we discard data points $d_i$ which have a lower rank than the boundary point by setting the $G(x_{val},\ w_i)$ as 0.
    
    \item \textbf{Monte Carlo steps} 
    Finally, we use the Monte Carlo method to approximate the value $G(x_{val},\ w_i)$ for all the remaining points.
    
\end{enumerate}

\begin{algorithm}[tb]
   \caption{$(\epsilon, \delta)$-approximation Algorithm for Gradients of Models with a General Utility Function}
   \label{alg:MCforMLE}
\begin{algorithmic}
   \STATE {\bfseries Input:} 
   $\mathcal{D}_{ret} = \{d_1, \cdots, d_M\}$, retrieval corpus;
   $\mathcal{D}_{val} = \{x_1,\cdots ,x_N\}$, validation set;
   $W = \{w_1, \cdots, w_M\}$, weights of \datapoint{s};
   $\epsilon,\delta$, error bound;
   \STATE {\bfseries Output:} \\
   $\{g_1,\cdots,g_M\}$, gradients of weights;\\
   \STATE {$\{g_1, \cdots, g_M\} \leftarrow 0$}
   \FOR{ $x_{val} \in \mathcal{D}_{val}$}
   \STATE $\{d_{\pi_1},\cdots,d_{\pi_M}\} \leftarrow$ SortByRankingScore($\mathcal{D}_{ret},x_{val}$)
   \STATE $b \leftarrow$ BinarySearch($\exp(-\frac{(\mu_b-K+1)^2}{2\mu_b})< \epsilon, \mu_b>K-1$)
   \STATE $T\leftarrow \lceil \frac{2}{\epsilon^2}\log(\frac{2 N}{\delta}) \rceil$
   \FOR{$i \leftarrow 1 \ to \ b$}
   \FOR{$t \leftarrow 1 \ to \ T$}
   \STATE $\mathcal{S} \leftarrow Sample(E,\mathcal{D}_{ret}\backslash d_{\pi_i})$
   \STATE {$\phi_t \leftarrow (U_{x_{val}}(\mathcal{S} \cup \{d_{\pi_i}\}) - U_{x_{val}}(\mathcal{S}))$}
   \ENDFOR
   \STATE $g_{\pi_i}\leftarrow g_{\pi_i} + \frac{1}{T}\cdot \frac{1}{N} \cdot\sum_{t=1}^{T}\phi_t$
   \ENDFOR
   \ENDFOR
\end{algorithmic}
\end{algorithm}

So far, we have finished the $\left(\epsilon,\delta\right)$-approximation for all gradient values $G(x_{val},\ w_i)$. For data points $d_i$ which has a lower rank than the boundary point, the approximate value equals $0$ because the $G(x_{val},\ w_i)$ is bounded by $\epsilon$. For data points which has a higher rank than the boundary point, the $\left(\epsilon,\delta\right)$-approximation is guaranteed by the Monte Carlo method. The pseudocode for the algorithm is Algorithm~\ref{alg:MCforMLE}.

The time complexity for computing the approximated gradient values is $\mathcal{O}(NM\log{M} + NMTC)$, where $N$ is the size of the validation set, $M$ is the size of the retrieval corpus, $T$ is the number of experiments conducted by the Monte Carlo Method and $C$ is the time complexity of each utility function evaluation. The time complexity of the improved algorithm is $\mathcal{O}(NM\log{M} + NBTC)$. $B$ is the index of the boundary point. With~\Cref{ap_BoundD}, we can see that if the value of all $w_i$ is larger than a certain constant $\lambda$, the overall time complexity is $\mathcal{O}(NM\log{M} + NKTC)$.


\section{Proofs and Details}
\subsection{Details of~\Cref{ap_computeep}}
\label{pcomputeep}
Suppose we can get the $\epsilon$-approximation $\hat{G}(x_{val},\ w_i)$ for the each $G(x_{val},\ w_i)$, then we will calculate $\hat{g}_i$ as the approximation for $\frac{\partial \tilde{U}}{\partial w_i}$: 
\begin{equation}
\hat{g}_i = \frac{1}{|\mathcal{D}_{val}|} \cdot \sum_{x_{val} \in \mathcal{D}_{val}}
\hat{G}(x_{val},\ w_i)    
\end{equation}
The difference between $\hat{g}_i$ and $\frac{\partial \tilde{U}}{\partial w_i}$ can be bound by:
\begin{equation}
\begin{aligned}
 \lvert \frac{\partial \tilde{U}}{\partial w_i} - \hat{g}_i\rvert 
& = \frac{1}{|\mathcal{D}_{val}|} \cdot \lvert \sum_{x_{val} \in \mathcal{D}_{val}}
G(x_{val},\ w_i) - \sum_{x_{val} \in \mathcal{D}_{val}}
\hat{G}(x_{val},\ w_i)\rvert \\
& \leq \frac{1}{|\mathcal{D}_{val}|} \cdot  \sum_{x_{val} \in \mathcal{D}_{val}} \lvert
G(x_{val},\ w_i) -
\hat{G}(x_{val},\ w_i)\rvert 
 \leq \frac{1}{|\mathcal{D}_{val}|} \cdot |\mathcal{D}_{val}| \cdot \epsilon = \epsilon\\
\end{aligned}
\end{equation}

\subsection{Details of~\Cref{ap_discard}}
\label{pdiscard}
Suppose $\mathcal{D}_{ret}^{\prime} = \{d_{b+1}, d_{b+2}, ..., d_M\}$ and $\mathcal{W}^{\prime} = \{w_{b+1}, w_{b+2}, ..., w_M\}$, we have:

\begin{equation}
\begin{aligned}
& \lvert \hat{G}(x_{val},\ w_i) - G(x_{val},\ w_i)\rvert = \lvert \mathcal{A}(x_{val},\mathcal{D}_{ret}\backslash \mathcal{D}_{ret}^{\prime},\mathcal{W}\backslash \mathcal{W}^{\prime}) - \mathcal{A}(x_{val},\mathcal{D}_{ret}, W )\rvert\\
& = \left| \sum_{\mathcal{S} \subseteq \mathcal{D}_{ret}\backslash \left(\mathcal{D}_{ret}^{\prime}\cup\{d_i\}\right)}
   \phi_{\mathcal{S}, x_{val}}\left(\ d_i  \right)
   \cdot P[\mathcal{S}]  - \sum_{\mathcal{S} \subseteq \mathcal{D}_{ret}\backslash d_i}
   \phi_{\mathcal{S}, x_{val}}\left(\ d_i  \right)
   \cdot P[\mathcal{S}] \right| \\   
& \leq \left| \sum_{\mathcal{S} \subseteq \mathcal{D}_{ret}\backslash d_i} \mathbb{I}
   \left\{ \mathcal{D}_{ret}^{\prime} \cap \operatorname{top_K}(\mathcal{S} ) \neq \emptyset \right\}
   \cdot P[\mathcal{S}] \right|  \leq P[{\bf PB}(w_1, w_2, \cdots, w_{b}) \leq K-1] \leq \epsilon
\end{aligned}
\end{equation}

\subsection{Details of~\Cref{ap_BoundD}}
\label{pBoundD}
$B$ is the minimum $b$ such that $\exp(-\frac{(\mu_{b}-K+1)^2}{2\mu_{b}}) < \epsilon$ and $\mu_{b} > K-1$. If the value of all $w_i$ is greater than a certain constant $\lambda$, suppose $b > \frac{4}{\lambda}\log{\frac{1}{\epsilon}} + \frac{2K - 2}{\lambda}$, we have:
\begin{equation}
\begin{aligned}
\mu_{b} > b\cdot\lambda > \frac{2K - 2}{\lambda} \cdot\lambda>2(K-1)
\end{aligned}
\end{equation}
and
\begin{equation}
\begin{aligned}
&\exp(-\frac{(\mu_{b}-K+1)^2}{2\mu_{b}}) < \exp(-\frac{(\mu_{b}-K+1)^2}{4(\mu_{b}-K+1)})
& < \exp(-\frac{\mu_{b}-K+1}{4}) < \exp(-\frac{4\log{\frac{1}{\epsilon}}}{4}) = \epsilon
\end{aligned}
\end{equation}

Therefore, $B$ is $\mathcal{O}(\frac{4}{\lambda}\log{\frac{1}{\epsilon}} + \frac{2K - 2}{\lambda})$. If we treat $\lambda$ and $\epsilon$ as contacts, $B$ is $\mathcal{O}(K)$.

\section{Full Results of Accuracy Added External Retrieval Source}
\label{gpt_jt_w_retrieve}

\begin{table*}
\centering
\caption{Accuracy using GPT-JT before and after adding external retrieval websites}
\label{rawgptretri}
\begin{center}
\scriptsize
\begin{sc}
\resizebox{\textwidth}{!}{
\begin{tabular}{lcccc|c}
	\toprule
    relation &
	\multirowcell{2}{GPT-JT (6B)\\w/o~Retrieval} &
	\multicolumn{3}{c}{GPT-JT (6B) w/ Retrieval} & \multirowcell{2}{\gptthree{}(175B)\\w/o~Retrieval} \\
	\cmidrule(r){3-5}
	  & &  K = 1 & K = 10 & K = 50 &\\
\midrule
average & 0.214 & 0.332 & \underline{0.333} & 0.293 & \textbf{0.339} \\
\hline
applies to jurisdiction & 0.430 & 0.620 & \underline{0.671} & 0.667 & \textbf{0.745} \\
author & 0.058 & \textbf{0.694} & \underline{0.609} & 0.498 & 0.369 \\
award received & 0.033 & 0.113 & \textbf{0.138} & \underline{0.119} & 0.114 \\
basic form of government & \underline{0.173} & 0.119 & 0.113 & 0.110 & \textbf{0.259} \\
capital & 0.515 & 0.625 & \textbf{0.676} & 0.638 & \underline{0.656} \\
capital of & 0.129 & 0.306 & \underline{0.364} & 0.286 & \textbf{0.381} \\
composer & 0.004 & \textbf{0.198} & 0.151 & 0.098 & \underline{0.168} \\
continent & 0.449 & \textbf{0.699} & \underline{0.674} & 0.615 & 0.442 \\
country & 0.624 & 0.758 & \underline{0.800} & \textbf{0.802} & 0.661 \\
country of citizenship & 0.515 & \underline{0.790} & \textbf{0.796} & 0.725 & 0.647 \\
country of origin & 0.461 & \underline{0.638} & \textbf{0.651} & 0.623 & 0.444 \\
creator & 0.019 & \textbf{0.422} & \underline{0.389} & 0.296 & 0.241 \\
currency & \underline{0.407} & 0.301 & 0.304 & 0.290 & \textbf{0.559} \\
developer & 0.199 & 0.386 & \underline{0.391} & 0.333 & \textbf{0.536} \\
director & 0.005 & \textbf{0.481} & \underline{0.321} & 0.110 & 0.236 \\
discoverer or inventor & 0.074 & 0.266 & \textbf{0.343} & \underline{0.313} & 0.192 \\
drug or therapy used for treatment & \underline{0.256} & 0.216 & 0.236 & 0.228 & \textbf{0.437} \\
educated at & 0.019 & \textbf{0.254} & \underline{0.152} & 0.044 & 0.114 \\
employer & 0.018 & \underline{0.291} & \textbf{0.308} & 0.146 & 0.238 \\
field of work & 0.070 & \underline{0.320} & \textbf{0.347} & 0.298 & 0.231 \\
genetic association & 0.015 & 0.031 & \textbf{0.053} & 0.049 & \underline{0.049} \\
genre & 0.105 & \textbf{0.227} & \underline{0.165} & 0.107 & 0.135 \\
has part & \textbf{0.071} & 0.006 & 0.005 & 0.010 & \underline{0.042} \\
head of government & 0.041 & 0.205 & \underline{0.228} & 0.217 & \textbf{0.409} \\
head of state & 0.220 & 0.245 & \underline{0.283} & 0.278 & \textbf{0.487} \\
headquarters location & 0.220 & \textbf{0.497} & \underline{0.496} & 0.444 & 0.424 \\
industry & 0.182 & 0.213 & \underline{0.218} & 0.215 & \textbf{0.225} \\
influenced by & 0.032 & \underline{0.122} & 0.069 & 0.038 & \textbf{0.147} \\
instance of & \underline{0.435} & 0.316 & 0.430 & \textbf{0.473} & 0.425 \\
instrument & 0.121 & \textbf{0.533} & \underline{0.514} & 0.370 & 0.246 \\
language of work or name & 0.534 & 0.741 & \underline{0.794} & \textbf{0.799} & 0.609 \\
languages spoken written or signed & 0.799 & 0.868 & \textbf{0.900} & \underline{0.895} & 0.444 \\
located in the administrative territorial entity & 0.088 & 0.137 & \textbf{0.147} & 0.132 & \underline{0.145} \\
location & \underline{0.194} & 0.137 & 0.166 & 0.189 & \textbf{0.349} \\
location of discovery & 0.040 & \underline{0.187} & \textbf{0.213} & 0.186 & 0.186 \\
location of formation & 0.127 & \underline{0.349} & 0.312 & 0.210 & \textbf{0.360} \\
majority opinion by & 0.070 & \underline{0.093} & 0.067 & 0.061 & \textbf{0.196} \\
manufacturer & 0.265 & \underline{0.332} & \textbf{0.343} & 0.316 & 0.329 \\
measured physical quantity & 0.271 & \underline{0.435} & 0.334 & 0.324 & \textbf{0.483} \\
medical condition treated & 0.193 & 0.277 & 0.344 & \underline{0.352} & \textbf{0.440} \\
member of & 0.110 & \underline{0.155} & 0.133 & 0.107 & \textbf{0.225} \\
member of political party & 0.213 & \textbf{0.481} & \underline{0.432} & 0.298 & 0.326 \\
member of sports team & 0.010 & \textbf{0.280} & \underline{0.230} & 0.156 & 0.109 \\
movement & 0.052 & 0.163 & \textbf{0.187} & \underline{0.185} & 0.118 \\
named after & \underline{0.292} & 0.156 & 0.199 & 0.203 & \textbf{0.452} \\
native language & 0.720 & \underline{0.834} & \textbf{0.846} & 0.819 & 0.627 \\
occupation & 0.187 & \textbf{0.524} & \underline{0.521} & 0.328 & 0.164 \\
office held by head of government & 0.066 & 0.103 & \underline{0.106} & 0.102 & \textbf{0.581} \\
official language & \underline{0.630} & 0.328 & 0.426 & 0.457 & \textbf{0.691} \\
operating system & \textbf{0.241} & 0.181 & 0.161 & 0.157 & \underline{0.228} \\
original language of film or TV show & \underline{0.625} & \textbf{0.626} & 0.614 & 0.564 & 0.498 \\
original network & 0.108 & \textbf{0.398} & 0.323 & 0.209 & \underline{0.380} \\
owned by & \underline{0.169} & 0.127 & 0.135 & 0.111 & \textbf{0.409} \\
part of & 0.099 & \underline{0.159} & 0.156 & 0.151 & \textbf{0.316} \\
participating team & 0.150 & \underline{0.231} & 0.183 & 0.126 & \textbf{0.340} \\
place of birth & 0.067 & \underline{0.386} & \textbf{0.482} & 0.382 & 0.171 \\
place of death & 0.109 & \textbf{0.390} & \underline{0.383} & 0.264 & 0.202 \\
position held & 0.114 & \underline{0.193} & \textbf{0.214} & 0.160 & 0.132 \\
position played on team & 0.034 & \textbf{0.465} & \underline{0.464} & 0.352 & 0.229 \\
programming language & 0.328 & \underline{0.498} & 0.489 & 0.458 & \textbf{0.642} \\
recommended unit of measurement & \underline{0.522} & 0.414 & 0.468 & 0.461 & \textbf{0.740} \\
record label & 0.020 & \textbf{0.176} & 0.102 & 0.043 & \underline{0.133} \\
religion & 0.397 & 0.466 & \textbf{0.486} & \underline{0.473} & 0.446 \\
shares border with & 0.003 & \underline{0.023} & 0.004 & 0.000 & \textbf{0.051} \\
stock exchange & 0.320 & 0.481 & \underline{0.528} & 0.472 & \textbf{0.761} \\
subclass of & \underline{0.179} & 0.110 & 0.100 & 0.114 & \textbf{0.201} \\
subsidiary & 0.052 & \underline{0.071} & 0.055 & 0.057 & \textbf{0.166} \\
symptoms and signs & 0.330 & 0.250 & \underline{0.341} & \textbf{0.367} & 0.284 \\
twinned administrative body & 0.003 & \underline{0.007} & 0.000 & 0.000 & \textbf{0.014} \\
work location & \underline{0.317} & 0.099 & 0.036 & 0.026 & \textbf{0.324} \\
\bottomrule
\end{tabular}
}
\end{sc}
\end{center}
\vskip -0.1in
\end{table*}

\section{Full Results of Weight-based Reweighting and Pruning}
\label{MLE_wikifact}
\begin{table*}
\centering
\caption{Accuracy for question answering on the 70 relations from \texttt{WikiFact}.}
\label{URLreweighting}
\begin{center}
\scriptsize
\begin{sc}
    \begin{tabular}{lcccc|c}
    \toprule
    \multirow{2}*{Relation}  &
    \multicolumn{4}{c}{GPT-JT (6B) w/~ Retrieval}&
    \multirow{2}*{\gptthree{}(175B) w/o~Retrieval} \\
    \cmidrule(r){2-5}
     &  vanilla & + LOO & + reweight & + prune&\\
    \midrule
    average & 0.333 & 0.358 & \underline{0.380} & \textbf{0.392} & 0.339\\
    \hline
    applies to jurisdiction & 0.671 & 0.682 & 0.715 & \underline{0.719} & \textbf{0.745}\\
    author & 0.609 & 0.669 & \underline{0.690} & \textbf{0.714} & 0.369\\
    award received & 0.138 & 0.146 & \underline{0.162} & \textbf{0.166} & 0.114\\
    basic form of government & 0.113 & 0.115 & 0.126 & \underline{0.141} & \textbf{0.259}\\
    capital & 0.676 & 0.689 & \underline{0.707} & \textbf{0.714} & 0.656\\
    capital of & 0.364 & 0.384 & \underline{0.415} & \textbf{0.427} & 0.381\\
    composer & 0.151 & 0.196 & \underline{0.204} & \textbf{0.236} & 0.168\\
    continent & 0.674 & 0.703 & \underline{0.748} & \textbf{0.752} & 0.442\\
    country & 0.800 & 0.808 & \underline{0.837} & \textbf{0.842} & 0.661\\
    country of citizenship & 0.796 & 0.821 & \underline{0.837} & \textbf{0.841} & 0.647\\
    country of origin & 0.651 & 0.689 & \underline{0.714} & \textbf{0.718} & 0.444\\
    creator & 0.389 & 0.422 & \underline{0.443} & \textbf{0.464} & 0.241\\
    currency & 0.304 & 0.312 & 0.356 & \underline{0.366} & \textbf{0.559}\\
    developer & 0.391 & 0.428 & 0.453 & \underline{0.465} & \textbf{0.536}\\
    director & 0.321 & \underline{0.441} & 0.433 & \textbf{0.474} & 0.236\\
    discoverer or inventor & 0.343 & 0.381 & \underline{0.383} & \textbf{0.414} & 0.192\\
    drug or therapy used for treatment & 0.236 & 0.245 & 0.257 & \underline{0.260} & \textbf{0.437}\\
    educated at & 0.152 & 0.189 & \underline{0.202} & \textbf{0.212} & 0.114\\
    employer & 0.308 & 0.334 & \underline{0.371} & \textbf{0.380} & 0.238\\
    field of work & 0.347 & 0.366 & \underline{0.401} & \textbf{0.419} & 0.231\\
    genetic association & 0.053 & 0.051 & \underline{0.057} & \textbf{0.058} & 0.049\\
    genre & 0.165 & \underline{0.238} & 0.234 & \textbf{0.256} & 0.135\\
    has part & 0.005 & 0.009 & 0.013 & \underline{0.020} & \textbf{0.042}\\
    head of government & 0.228 & 0.231 & 0.272 & \underline{0.301} & \textbf{0.409}\\
    head of state & 0.283 & 0.288 & 0.324 & \underline{0.332} & \textbf{0.487}\\
    headquarters location & 0.496 & 0.543 & \underline{0.563} & \textbf{0.575} & 0.424\\
    industry & 0.218 & 0.229 & \underline{0.250} & \textbf{0.257} & 0.225\\
    influenced by & 0.069 & 0.094 & 0.088 & \underline{0.111} & \textbf{0.147}\\
    instance of & 0.430 & 0.463 & \underline{0.536} & \textbf{0.547} & 0.425\\
    instrument & 0.514 & 0.558 & \underline{0.619} & \textbf{0.631} & 0.246\\
    language of work or name & 0.794 & 0.804 & \underline{0.832} & \textbf{0.838} & 0.609\\
    languages spoken written or signed & 0.900 & 0.905 & \underline{0.916} & \textbf{0.917} & 0.444\\
    located in the administrative territorial entity & 0.147 & 0.153 & \underline{0.167} & \textbf{0.173} & 0.145\\
    location & 0.166 & 0.179 & 0.203 & \underline{0.208} & \textbf{0.349}\\
    location of discovery & 0.213 & 0.222 & \underline{0.245} & \textbf{0.257} & 0.186\\
    location of formation & 0.312 & \underline{0.370} & 0.364 & \textbf{0.379} & 0.360\\
    majority opinion by & 0.067 & 0.078 & 0.077 & \underline{0.091} & \textbf{0.196}\\
    manufacturer & 0.343 & 0.355 & \underline{0.377} & \textbf{0.385} & 0.329\\
    measured physical quantity & 0.334 & 0.380 & 0.400 & \underline{0.402} & \textbf{0.483}\\
    medical condition treated & 0.344 & 0.367 & 0.402 & \underline{0.407} & \textbf{0.440}\\
    member of & 0.133 & 0.143 & 0.173 & \underline{0.193} & \textbf{0.225}\\
    member of political party & 0.432 & 0.459 & \underline{0.483} & \textbf{0.491} & 0.326\\
    member of sports team & 0.230 & 0.309 & \underline{0.330} & \textbf{0.345} & 0.109\\
    movement & 0.187 & 0.195 & \underline{0.240} & \textbf{0.249} & 0.118\\
    named after & 0.199 & 0.216 & 0.249 & \underline{0.256} & \textbf{0.452}\\
    native language & 0.846 & 0.855 & \underline{0.865} & \textbf{0.866} & 0.627\\
    occupation & 0.521 & 0.553 & \underline{0.571} & \textbf{0.591} & 0.164\\
    office held by head of government & 0.106 & 0.111 & 0.121 & \underline{0.122} & \textbf{0.581}\\
    official language & 0.426 & 0.470 & 0.536 & \underline{0.558} & \textbf{0.691}\\
    operating system & 0.161 & 0.183 & 0.185 & \underline{0.206} & \textbf{0.228}\\
    original language of film or TV show & 0.614 & 0.687 & \underline{0.696} & \textbf{0.711} & 0.498\\
    original network & 0.323 & 0.389 & \underline{0.394} & \textbf{0.404} & 0.380\\
    owned by & 0.135 & 0.138 & 0.156 & \underline{0.171} & \textbf{0.409}\\
    part of & 0.156 & 0.180 & 0.181 & \underline{0.200} & \textbf{0.316}\\
    participating team & 0.183 & 0.188 & 0.198 & \underline{0.200} & \textbf{0.340}\\
    place of birth & 0.482 & 0.501 & \underline{0.533} & \textbf{0.536} & 0.171\\
    place of death & 0.383 & 0.407 & \underline{0.445} & \textbf{0.456} & 0.202\\
    position held & 0.214 & 0.236 & \underline{0.270} & \textbf{0.277} & 0.132\\
    position played on team & 0.464 & 0.517 & \underline{0.550} & \textbf{0.560} & 0.229\\
    programming language & 0.489 & 0.493 & 0.528 & \underline{0.532} & \textbf{0.642}\\
    recommended unit of measurement & 0.468 & 0.471 & 0.484 & \underline{0.486} & \textbf{0.740}\\
    record label & 0.102 & \underline{0.164} & 0.159 & \textbf{0.172} & 0.133\\
    religion & 0.486 & 0.490 & \underline{0.535} & \textbf{0.545} & 0.446\\
    shares border with & 0.004 & 0.014 & 0.012 & \underline{0.022} & \textbf{0.051}\\
    stock exchange & 0.528 & 0.563 & 0.614 & \underline{0.616} & \textbf{0.761}\\
    subclass of & 0.100 & 0.120 & 0.156 & \underline{0.183} & \textbf{0.201}\\
    subsidiary & 0.055 & 0.072 & 0.076 & \underline{0.094} & \textbf{0.166}\\
    symptoms and signs & 0.341 & 0.364 & \underline{0.407} & \textbf{0.409} & 0.284\\
    twinned administrative body & 0.000 & 0.001 & 0.001 & \textbf{0.015} & \underline{0.014}\\
    work location & 0.036 & 0.061 & 0.051 & \underline{0.084} & \textbf{0.324}\\
    \bottomrule
    \end{tabular}
\end{sc}
\end{center}
\vskip -0.1in
\end{table*}

\section{Full Results on Dirty Corpus}
\begin{table*}
\centering
\caption{Accuracy improvements for GPT-JT (6B) with retrieval augmentation on a noisy corpus.}
\label{dirtycorpus}
\begin{center}
\scriptsize

\begin{sc}
\begin{tabular}{lccccr}
\toprule
\multirow{2}*{Relation} &
\multirow{2}*{CLEAN CORPUS} &
\multicolumn{4}{c}{DIRTY CORPUS} \\ 
\cmidrule{3-6}
 & & vanilla & + LOO & + reweight & + prune\\
\midrule
average & \underline{0.333} & 0.270 & 0.311 & 0.330 & \textbf{0.335} \\
\hline
applies to jurisdiction & \textbf{0.671} & 0.606 & 0.648 & 0.651 & \underline{0.659} \\
author & 0.609 & 0.364 & 0.562 & \underline{0.610} & \textbf{0.614} \\
award received & \textbf{0.138} & 0.095 & 0.106 & 0.108 & \underline{0.116} \\
basic form of government & 0.113 & 0.113 & 0.121 & \underline{0.122} & \textbf{0.122} \\
capital & \textbf{0.676} & 0.571 & 0.622 & 0.663 & \underline{0.667} \\
capital of & \textbf{0.364} & 0.215 & 0.271 & 0.327 & \underline{0.335} \\
composer & 0.151 & 0.098 & 0.142 & \underline{0.180} & \textbf{0.180} \\
continent & 0.674 & \textbf{0.698} & \underline{0.697} & \textbf{0.698} & 0.697 \\
country & 0.800 & 0.696 & 0.778 & \underline{0.801} & \textbf{0.801} \\
country of citizenship & 0.796 & 0.744 & 0.777 & \underline{0.797} & \textbf{0.799} \\
country of origin & \textbf{0.651} & 0.575 & 0.618 & 0.642 & \underline{0.642} \\
creator & 0.389 & 0.245 & 0.352 & \underline{0.393} & \textbf{0.395} \\
currency & \textbf{0.304} & 0.277 & 0.291 & 0.299 & \underline{0.300} \\
developer & \textbf{0.391} & 0.283 & 0.350 & 0.376 & \underline{0.378} \\
director & 0.321 & 0.258 & 0.340 & \underline{0.423} & \textbf{0.423} \\
discoverer or inventor & \textbf{0.343} & 0.123 & 0.237 & 0.246 & \underline{0.275} \\
drug or therapy used for treatment & \textbf{0.236} & 0.187 & 0.217 & \underline{0.234} & 0.234 \\
educated at & 0.152 & 0.104 & 0.154 & \textbf{0.195} & \underline{0.195} \\
employer & \textbf{0.308} & 0.164 & 0.232 & 0.301 & \underline{0.304} \\
field of work & \textbf{0.347} & 0.288 & 0.310 & 0.320 & \underline{0.331} \\
genetic association & \textbf{0.053} & 0.017 & 0.032 & 0.043 & \underline{0.044} \\
genre & 0.165 & 0.181 & \textbf{0.193} & 0.184 & \underline{0.187} \\
has part & 0.005 & \textbf{0.009} & 0.008 & 0.008 & \underline{0.008} \\
head of government & \textbf{0.228} & 0.146 & 0.200 & 0.208 & \underline{0.213} \\
head of state & \textbf{0.283} & 0.196 & 0.235 & 0.251 & \underline{0.265} \\
headquarters location & \textbf{0.496} & 0.410 & 0.457 & 0.486 & \underline{0.490} \\
industry & \textbf{0.218} & 0.190 & 0.196 & 0.194 & \underline{0.209} \\
influenced by & 0.069 & 0.052 & 0.085 & \textbf{0.089} & \underline{0.089} \\
instance of & \underline{0.430} & 0.289 & 0.428 & 0.419 & \textbf{0.438} \\
instrument & 0.514 & 0.425 & 0.470 & \underline{0.525} & \textbf{0.529} \\
language of work or name & \textbf{0.794} & 0.748 & 0.776 & 0.756 & \underline{0.789} \\
languages spoken written or signed & \textbf{0.900} & 0.854 & 0.879 & 0.891 & \underline{0.894} \\
located in the administrative territorial entity & \textbf{0.147} & 0.109 & 0.133 & \underline{0.146} & 0.145 \\
location & \textbf{0.166} & 0.116 & 0.139 & 0.140 & \underline{0.148} \\
location of discovery & \textbf{0.213} & 0.119 & 0.157 & 0.194 & \underline{0.196} \\
location of formation & 0.312 & 0.245 & 0.289 & \underline{0.334} & \textbf{0.335} \\
majority opinion by & 0.067 & 0.067 & 0.071 & \textbf{0.073} & \underline{0.072} \\
manufacturer & \textbf{0.343} & 0.263 & 0.321 & 0.341 & \underline{0.341} \\
measured physical quantity & 0.334 & 0.341 & 0.367 & \underline{0.377} & \textbf{0.378} \\
medical condition treated & \textbf{0.344} & 0.253 & 0.306 & 0.307 & \underline{0.327} \\
member of & 0.133 & 0.131 & 0.145 & \textbf{0.145} & \underline{0.145} \\
member of political party & 0.432 & 0.360 & 0.424 & \underline{0.450} & \textbf{0.451} \\
member of sports team & 0.230 & 0.142 & 0.225 & \underline{0.231} & \textbf{0.234} \\
movement & \textbf{0.187} & 0.130 & 0.158 & 0.156 & \underline{0.177} \\
named after & \textbf{0.199} & 0.119 & 0.159 & 0.179 & \underline{0.183} \\
native language & \textbf{0.846} & 0.825 & 0.836 & 0.842 & \underline{0.846} \\
occupation & 0.521 & 0.428 & 0.465 & \underline{0.523} & \textbf{0.524} \\
office held by head of government & \textbf{0.106} & 0.099 & 0.097 & 0.100 & \underline{0.103} \\
official language & \textbf{0.426} & 0.315 & 0.406 & 0.396 & \underline{0.417} \\
operating system & 0.161 & 0.162 & \textbf{0.163} & 0.163 & \underline{0.163} \\
original language of film or TV show & 0.614 & 0.554 & 0.596 & \textbf{0.631} & \underline{0.631} \\
original network & 0.323 & 0.248 & 0.318 & \textbf{0.358} & \underline{0.358} \\
owned by & \textbf{0.135} & 0.093 & 0.115 & 0.127 & \underline{0.127} \\
part of & \textbf{0.156} & 0.124 & 0.133 & 0.142 & \underline{0.144} \\
participating team & 0.183 & 0.122 & 0.183 & \underline{0.190} & \textbf{0.191} \\
place of birth & \textbf{0.482} & 0.272 & 0.377 & 0.455 & \underline{0.460} \\
place of death & 0.383 & 0.269 & 0.324 & \underline{0.393} & \textbf{0.394} \\
position held & \textbf{0.214} & 0.175 & 0.184 & 0.200 & \underline{0.206} \\
position played on team & 0.464 & 0.398 & 0.444 & \underline{0.491} & \textbf{0.492} \\
programming language & \textbf{0.489} & 0.430 & 0.466 & 0.478 & \underline{0.478} \\
recommended unit of measurement & \textbf{0.468} & 0.385 & 0.418 & 0.429 & \underline{0.442} \\
record label & 0.102 & 0.084 & 0.112 & \underline{0.122} & \textbf{0.123} \\
religion & \textbf{0.486} & 0.438 & 0.461 & 0.466 & \underline{0.471} \\
shares border with & 0.004 & 0.007 & 0.008 & \underline{0.010} & \textbf{0.010} \\
stock exchange & \textbf{0.528} & 0.402 & 0.462 & 0.519 & \underline{0.519} \\
subclass of & 0.100 & 0.101 & 0.112 & \underline{0.113} & \textbf{0.117} \\
subsidiary & 0.055 & 0.057 & \textbf{0.059} & \underline{0.059} & 0.058 \\
symptoms and signs & \textbf{0.341} & 0.228 & 0.310 & 0.316 & \underline{0.333} \\
twinned administrative body & 0.000 & 0.002 & \underline{0.002} & \textbf{0.002} & 0.002 \\
work location & 0.036 & 0.060 & \textbf{0.066} & \underline{0.066} & 0.065 \\
\bottomrule
\end{tabular}
\end{sc}
\end{center}
\vskip -0.1in
\end{table*}

\section{Full Results on Additional Fabricated Data}
\begin{table*}
\centering
\caption{Accuracy impact of additional fabricated data sources for question answering on \texttt{Wikifact}.}
\label{fakewikireweighting}
\begin{center}
\tiny
\begin{sc}
\begin{tabular}{lccccc|c}
\toprule
\multirow{2}*{Relation}  &
\multirow{2}*{GPT-JT (6B) w/~ Retrieval} &
\multicolumn{4}{c}{GPT-JT (6B) w/~ Retrieval + Fabricated Data}&
\multirow{2}*{\gptthree{} (175B) w/o~Retrieval} \\ 
\cmidrule{3-6}
 & &  vanilla & +LOO & +reweight & +prune&\\
\midrule
average & 0.333 & 0.382 & 0.399 & \underline{0.410} & \textbf{0.418} & 0.339 \\
\hline
applies to jurisdiction & 0.671 & 0.693 & 0.692 & 0.714 & \underline{0.721} & \textbf{0.745} \\
author & 0.609 & 0.660 & 0.662 & \underline{0.691} & \textbf{0.707} & 0.369 \\
award received & 0.138 & 0.120 & 0.139 & \underline{0.162} & \textbf{0.168} & 0.114 \\
basic form of government & 0.113 & 0.139 & \underline{0.233} & 0.167 & 0.182 & \textbf{0.259} \\
capital & 0.676 & 0.707 & \underline{0.707} & 0.707 & \textbf{0.715} & 0.656 \\
capital of & 0.364 & 0.417 & 0.417 & \underline{0.436} & \textbf{0.442} & 0.381 \\
composer & 0.151 & 0.247 & 0.247 & \underline{0.259} & \textbf{0.274} & 0.168 \\
continent & 0.674 & 0.843 & 0.846 & \underline{0.854} & \textbf{0.855} & 0.442 \\
country & 0.800 & 0.821 & 0.819 & \underline{0.837} & \textbf{0.842} & 0.661 \\
country of citizenship & 0.796 & 0.802 & 0.823 & \underline{0.837} & \textbf{0.842} & 0.647 \\
country of origin & 0.651 & 0.677 & 0.681 & \underline{0.713} & \textbf{0.719} & 0.444 \\
creator & 0.389 & \underline{0.464} & 0.461 & 0.443 & \textbf{0.466} & 0.241 \\
currency & 0.304 & 0.272 & 0.305 & 0.355 & \underline{0.364} & \textbf{0.559} \\
developer & 0.391 & 0.472 & 0.475 & 0.499 & \underline{0.503} & \textbf{0.536} \\
director & 0.321 & 0.476 & 0.492 & \underline{0.497} & \textbf{0.503} & 0.236 \\
discoverer or inventor & 0.343 & 0.271 & 0.375 & \underline{0.378} & \textbf{0.406} & 0.192 \\
drug or therapy used for treatment & 0.236 & 0.306 & 0.308 & 0.318 & \underline{0.320} & \textbf{0.437} \\
educated at & 0.152 & \underline{0.210} & 0.207 & 0.202 & \textbf{0.213} & 0.114 \\
employer & 0.308 & 0.360 & 0.359 & \underline{0.371} & \textbf{0.381} & 0.238 \\
field of work & 0.347 & 0.274 & 0.367 & \underline{0.401} & \textbf{0.419} & 0.231 \\
genetic association & 0.053 & 0.053 & 0.052 & \underline{0.057} & \textbf{0.057} & 0.049 \\
genre & 0.165 & 0.212 & 0.234 & \underline{0.244} & \textbf{0.264} & 0.135 \\
has part & 0.005 & 0.015 & 0.014 & 0.016 & \underline{0.018} & \textbf{0.042} \\
head of government & 0.228 & 0.402 & 0.404 & \underline{0.416} & \textbf{0.422} & 0.409 \\
head of state & 0.283 & 0.437 & 0.439 & 0.448 & \underline{0.450} & \textbf{0.487} \\
headquarters location & 0.496 & 0.538 & 0.537 & \underline{0.566} & \textbf{0.576} & 0.424 \\
industry & 0.218 & 0.238 & 0.235 & \underline{0.250} & \textbf{0.259} & 0.225 \\
influenced by & 0.069 & 0.112 & 0.107 & 0.100 & \underline{0.112} & \textbf{0.147} \\
instance of & 0.430 & 0.388 & 0.459 & \underline{0.535} & \textbf{0.547} & 0.425 \\
instrument & 0.514 & 0.615 & \underline{0.621} & 0.618 & \textbf{0.628} & 0.246 \\
language of work or name & 0.794 & 0.809 & 0.808 & \underline{0.834} & \textbf{0.839} & 0.609 \\
languages spoken written or signed & 0.900 & \textbf{0.917} & \underline{0.917} & 0.916 & 0.917 & 0.444 \\
located in the administrative territorial entity & 0.147 & 0.153 & 0.154 & \underline{0.167} & \textbf{0.173} & 0.145 \\
location & 0.166 & 0.195 & 0.195 & 0.215 & \underline{0.221} & \textbf{0.349} \\
location of discovery & 0.213 & 0.198 & 0.215 & \underline{0.245} & \textbf{0.258} & 0.186 \\
location of formation & 0.312 & 0.386 & 0.385 & \underline{0.397} & \textbf{0.404} & 0.360 \\
majority opinion by & 0.067 & 0.111 & 0.128 & 0.125 & \underline{0.128} & \textbf{0.196} \\
manufacturer & 0.343 & 0.368 & 0.366 & \underline{0.377} & \textbf{0.385} & 0.329 \\
measured physical quantity & 0.334 & 0.550 & \textbf{0.584} & 0.562 & \underline{0.564} & 0.483 \\
medical condition treated & 0.344 & 0.383 & 0.400 & 0.402 & \underline{0.406} & \textbf{0.440} \\
member of & 0.133 & 0.190 & 0.191 & 0.203 & \underline{0.210} & \textbf{0.225} \\
member of political party & 0.432 & 0.478 & 0.478 & \underline{0.489} & \textbf{0.494} & 0.326 \\
member of sports team & 0.230 & 0.197 & 0.299 & \underline{0.330} & \textbf{0.341} & 0.109 \\
movement & 0.187 & 0.177 & 0.200 & \underline{0.240} & \textbf{0.250} & 0.118 \\
named after & 0.199 & 0.338 & 0.343 & 0.353 & \underline{0.358} & \textbf{0.452} \\
native language & 0.846 & 0.863 & 0.863 & \underline{0.867} & \textbf{0.868} & 0.627 \\
occupation & 0.521 & 0.543 & 0.552 & \underline{0.571} & \textbf{0.592} & 0.164 \\
office held by head of government & 0.106 & 0.137 & 0.137 & \underline{0.145} & 0.145 & \textbf{0.581} \\
official language & 0.426 & 0.701 & 0.711 & \underline{0.722} & \textbf{0.727} & 0.691 \\
operating system & 0.161 & 0.207 & \underline{0.214} & 0.185 & 0.206 & \textbf{0.228} \\
original language of film or TV show & 0.614 & 0.771 & 0.772 & \underline{0.779} & \textbf{0.781} & 0.498 \\
original network & 0.323 & 0.454 & 0.467 & \underline{0.483} & \textbf{0.485} & 0.380 \\
owned by & 0.135 & 0.195 & 0.192 & 0.201 & \underline{0.211} & \textbf{0.409} \\
part of & 0.156 & 0.192 & 0.199 & 0.199 & \underline{0.205} & \textbf{0.316} \\
participating team & 0.183 & 0.270 & 0.273 & 0.281 & \underline{0.281} & \textbf{0.340} \\
place of birth & 0.482 & 0.360 & 0.491 & \underline{0.533} & \textbf{0.537} & 0.171 \\
place of death & 0.383 & 0.335 & 0.396 & \underline{0.445} & \textbf{0.455} & 0.202 \\
position held & 0.214 & 0.241 & 0.244 & \underline{0.270} & \textbf{0.279} & 0.132 \\
position played on team & 0.464 & 0.435 & 0.511 & \underline{0.549} & \textbf{0.559} & 0.229 \\
programming language & 0.489 & 0.705 & 0.707 & \underline{0.717} & \textbf{0.720} & 0.642 \\
recommended unit of measurement & 0.468 & 0.464 & 0.466 & 0.484 & \underline{0.487} & \textbf{0.740} \\
record label & 0.102 & 0.220 & 0.228 & \underline{0.246} & \textbf{0.250} & 0.133 \\
religion & 0.486 & 0.535 & 0.535 & \underline{0.549} & \textbf{0.555} & 0.446 \\
shares border with & 0.004 & 0.010 & \underline{0.016} & 0.011 & 0.016 & \textbf{0.051} \\
stock exchange & 0.528 & 0.754 & 0.766 & \underline{0.773} & \textbf{0.774} & 0.761 \\
subclass of & 0.100 & 0.129 & 0.151 & 0.155 & \underline{0.186} & \textbf{0.201} \\
subsidiary & 0.055 & 0.087 & 0.100 & 0.096 & \underline{0.105} & \textbf{0.166} \\
symptoms and signs & 0.341 & 0.348 & 0.363 & \underline{0.406} & \textbf{0.410} & 0.284 \\
twinned administrative body & 0.000 & 0.003 & 0.005 & 0.002 & \underline{0.005} & \textbf{0.014} \\
work location & 0.036 & 0.082 & \underline{0.167} & 0.096 & 0.107 & \textbf{0.324} \\
\bottomrule
\end{tabular}
\end{sc}

\end{center}
\vskip -0.1in
\end{table*}

\section{Full Results on OpenAi Generator}
\begin{table*}
\centering
\caption{Accuracy for question answering using OpenAI on the 5 relations from \texttt{WikiFact}.}
\label{openaianswer}
\begin{center}
\scriptsize
\begin{sc}
\resizebox{\textwidth}{!}{
\begin{tabular}{lccccc|ccccc}
\toprule
\multirow{2}*{Relation}  &
\multirow{2}*{GPT-JT (6B)~w/o} &\multicolumn{4}{c}{GPT-JT (6B) w/~ Retrieval}&
\multirow{2}*{\gptthree{} (175B)~w/o} 
&\multicolumn{4}{c}{\gptthree{} w/~ Retrieval}
\\ 
\cmidrule{3-6}
\cmidrule{8-11}
 & &  vanilla & +LOO & +reweight & +prune& &  vanilla & +LOO & +reweight & +prune\\
\midrule
applies to jurisdiction & 0.430 & 0.673 & 0.681 & 0.715 & 0.719 & 0.745 & 0.718 & 0.734 & \underline{0.750} & \textbf{0.751}\\
author & 0.058 & 0.615 & 0.669 & 0.690 & 0.714 & 0.369 & 0.712 & 0.724 & \underline{0.741} & \textbf{0.749}\\
award received & 0.033 & 0.137 & 0.146 & 0.162 & 0.166 & 0.114 & 0.179 & 0.183 & \underline{0.187} & \textbf{0.189}\\
basic form of government & 0.173 & 0.114 & 0.115 & 0.126 & 0.141 & 0.259 & 0.342 & 0.362 & \underline{0.396} & \textbf{0.409}\\
capital & 0.515 & 0.679 & 0.689 & 0.707 & 0.714 & 0.656 & 0.756 & 0.756 & \underline{0.769} & \textbf{0.775}\\
\hline
average & 0.242 & 0.444 & 0.460 & 0.480 & 0.491 & 0.429 & 0.541 & 0.552 & \underline{0.569} & \textbf{0.575}\\
\bottomrule
\end{tabular}
}
\end{sc}

\end{center}
\vskip -0.1in
\end{table*}



\end{document}